\documentclass[10pt,twocolumn,letterpaper]{article}

\usepackage[pagenumbers]{cvpr} 

\makeatletter
\@namedef{ver@everyshi.sty}{}
\makeatother
\usepackage{tikz}

\usepackage{graphicx}
\usepackage{amsmath}
\usepackage{amssymb}
\usepackage{booktabs}
\usepackage{svg}
\usepackage{wrapfig}
\usepackage{multirow}
\usepackage{makecell}

\newtheorem{lemma}{Lemma}
\newtheorem{theorem}{Theorem}

\newtheorem{assumption}{Assumption}

\usepackage{amsfonts}
\usepackage{float}
\usepackage{enumerate}
\usepackage{enumitem}

  \usepackage{algpseudocode}
\newcommand{\SV}[1]{{\color{blue} Saeed: #1}}

\usepackage{tabularx}

\newenvironment{proof}[1][Proof]{\textbf{#1.} }{\ \rule{0.5em}{0.5em} \vspace{1ex}}
\usepackage{booktabs} 

\usepackage{algcompatible}

\usepackage[linesnumbered,ruled,vlined]{algorithm2e}
\usepackage[pagebackref,breaklinks,colorlinks]{hyperref}

\usepackage[capitalize]{cleveref}
\crefname{section}{Sec.}{Secs.}
\Crefname{section}{Section}{Sections}
\Crefname{table}{Table}{Tables}
\crefname{table}{Tab.}{Tabs.}

\def\cvprPaperID{11498} 
\def\confName{CVPR}
\def\confYear{2023}

\begin{document}

\title{When Do Curricula Work in Federated Learning?}

\author{%
Saeed Vahidian$^{1}$, Sreevatsank Kadaveru$^{1}$, Woonjoon Baek$^{1}$, Weijia Wang$^{1}$,  Vyacheslav Kungurtsev$^2$,\\
Chen Chen$^3$, Mubarak Shah$^3$, Bill Lin$^{1}$\\\\
$^1$University of California San Diego \quad $^2$Czech Technical University \quad $^3$ UCF\\
}


\maketitle
  
\begin{abstract}
\vspace{-3mm}
 

An oft-cited open problem of federated learning is the existence of data heterogeneity at the clients. One pathway to understanding the drastic accuracy drop in federated learning is by scrutinizing the behavior of the clients' deep models on data with different levels of "difficulty", which has been left unaddressed. In this paper, we investigate a different and rarely studied dimension of FL: ordered learning. Specifically, we aim to investigate how ordered learning principles can contribute to alleviating the heterogeneity effects in FL. We present theoretical analysis and conduct extensive empirical studies on the efficacy of orderings spanning three kinds of learning: curriculum, anti-curriculum, and random curriculum. We find that curriculum learning largely alleviates non-IIDness. Interestingly, the more disparate the data distributions across clients the more they benefit from ordered learning. We provide analysis explaining this phenomenon, specifically indicating how curriculum training appears to make the objective landscape progressively less convex, suggesting fast converging iterations at the beginning of the training procedure. We derive quantitative results of convergence for both convex and nonconvex objectives by modeling the curriculum training on federated devices as local SGD with locally biased stochastic gradients. Also, inspired by ordered learning, we propose a novel client selection technique that benefits from the real-world disparity in the clients. Our proposed approach to client selection has a synergic effect when applied together with ordered learning in FL.

\end{abstract}


\vspace{-6mm}
\section{Introduction}
\label{sec:intro}

Inspired by the learning principle underlying the cognitive process of humans, curriculum learning (CL) generally proposes a training paradigm for machine learning models in which the difficulty of the training task is progressively scaled, going from "easy" to "hard". Prior empirical studies have demonstrated that CL is effective in avoiding bad local minima and in improving the generalization results~\cite{bengio-curriculum, SPL-curriculum-2015}. Also interestingly, another line of work proposes the exact opposite strategy of prioritizing the harder examples first, such as~\cite{curriculum-anti-2016-ijcai, anti-is-better-2018, anti-is-better-2019}--these techniques are referred to as ``anti-curriculum". It is shown that certain tasks can benefit from anti-curriculum techniques. However, in tasks such as object detection~\cite{ST3D-object-detection, curriculum-object-detection-2018-sangineto}, and large-scale text models~\cite{curriculum-text-2020} CL is standard practice.

Although the empirical observations on CL appear to be in conflict, this has not impeded the study of CL in machine learning tasks. Certain scenarios~\cite{curriculum-Neyshabur-2021} have witnessed the potential benefits of CL. The efficacy of CL has been explored in a considerable breadth of applications, including, but not limited to, supervised learning tasks within computer vision~\cite{CurriculumNet-2018}, healthcare~\cite{curriculum-federated-2021}, reinforcement learning tasks~\cite{curriculum-Reinforcement-learning-2018}, natural language processing (NLP)~\cite{curriculum-nlp-2019} as well as other applications such as graph learning~\cite{curriculum-graph-2019}, and neural architecture search~\cite{curriculum-NAS-2020}.


Curriculum learning has been studied in considerable depth for the standard centralized training settings. However, to the best of our knowledge, our paper is the first attempt at studying the methodologies, applications, and efficacy of CL in a decentralized training setting and in particular for federated learning (FL). In FL, the training time budget and the communication bandwidth are the key limiting constraints, and as demonstrated in~\cite{curriculum-Neyshabur-2021} CL is particularly effective in settings with a limited training budget. It is an interesting proposition to apply the CL idea to an FL setting, and that is exactly what we explore in our paper (in Section~\ref{sec:motivation}).

The idea of CL is agnostic to the underlying algorithms used for federation and hence can be very easily applied to any of the state-of-the-art solutions in FL. Our technique does not require a pre-trained expert model and does not impose any additional synchronization overhead on the system. Also, as the CL is applied to the client, it does not add any additional computational overhead to the server.


Further, we propose a novel framework for efficient client selection in an FL setting that builds upon our idea of CL in FL. We show in Section~\ref{curr-on-client}, CL on clients is able to leverage the real-world discrepancy in the clients to its advantage. Furthermore, when combined with the primary idea of CL in FL, we find that it provides compounding benefits.



\textbf{Contributions:} In this paper, we comprehensively assess the efficacy of CL in FL and provide novel insights into the efficacy of CL under the various conditions and methodologies in the FL environment. 

We provide a rigorous convergence theory and analysis of FL in non-IID settings, under strongly convex and non-convex assumptions, by considering local training steps as biased SGD, where CL naturally grows the bias over the iterations, in Section~\ref{analysis-main-paper} of the main paper and Section~\ref{sec:analysis} of the \textbf{Supplementary Material (SM)}. 

We hope to provide comprehensible answers to the following six important questions: \\

\setlist{nolistsep}

\noindent \textbf{Q1:} \emph{Which of the curriculum learning paradigm is effective in FL? And under what conditions?} 

\noindent \textbf{Q2:} \emph{Can CL alleviate the statistical data heterogeneity in FL?} 

 \noindent \textbf{Q3:} \emph{Does the efficacy of CL in FL depend on the underlying client data distributions?} 

\noindent \textbf{Q4:} \emph{Whether the effectiveness of CL is correlated with the size of datasets owned by each client?} 

\noindent \textbf{Q5:} \emph{Are there any benefits of smart client selection? And can CL be applied to the problem of client selection?} 

\noindent \textbf{Q6:} \emph{Can we apply the ideas of CL to both the client data and client selection?}\\

We test our ideas on two widely adopted network architectures on popular datasets in the FL community (CIFAR-10, CIFAR-100, and FMNIST) under a wide range of choices of curricula and compare them with several global state-of-the-art baselines. We have the following findings:\\

\begin{itemize}[noitemsep,leftmargin=*]


    \item CL in FL boosts the classiﬁcation accuracy under both IID and Non-IID data distributions (Sections~\ref{effect-scoring-function}, and~\ref{Effect-of-pacing-function}).

    
    
    \item The efficacy of CL is more pronounced when the client data is heterogeneous (Section~\ref{effect-of-hetero}).

     \item CL on client selection has a synergic effect that compounds the benefits of CL in FL (Section~\ref{curr-on-client}).
     
    \item CL can alleviate data heterogeneity across clients and CL is particularly effective in the initial stages of training as the larger initial steps of the optimization are done on the ``easier'' examples which are closer together in distribution (Section~\ref{analysis-main-paper} of main paper, and Section~\ref{sec:analysis} of the SM).
    
    \item   The efficacy of our technique is observed in both lower and higher data regimes (Section~\ref{Effect-of-amount-of-data} of the SM).
    
\end{itemize}

\section{Curriculum Components}\label{sec:motivation}

Federated Learning (FL) techniques provide a mechanism to address the growing privacy and security concerns associated with data collection, all-the-while satiating the need for large datasets for training powerful machine learning models. A major appeal of FL is its ability to train a model over a distributed dataset without needing the data to be collated into a central location for training. In the FL framework, we have a server and multiple clients with a distributed dataset. The process of federation is an iterative process that involves multiple rounds of back-and-forth communication between the server and the clients that participate in the process~\cite{Vahidian-rethinking-2022-workshop}. This back-and-forth communication incurs a significant communication overhead, thereby limiting the number of rounds that are pragmatically possible in real-world applications. Curriculum learning is an idea that particularly shines in these scenarios where the training time is limited~\cite{curriculum-Neyshabur-2021}. Motivated by this idea, we define a curriculum for training in the federated learning setting. A curriculum consists of three key components:

\noindent \textbf{The scoring function:} It is a mapping from an input sample, $x_i \in \mathcal{D}=\{(x_1, y_1), (x_2, y_2),..., (x_n, y_n)\}$, to a numerical value, $s_i(x_i) \in R^+$. We define a range of scoring functions when defining a CL for the FL setting in the subsequent sections. When defining the scoring function of a curriculum in FL, we look for loss-based dynamic measures for the score that update every iteration, unlike the methods proposed in~\cite{curriculum-cscore-2020} which produce a fixed score for each sample. This is because the instantaneous score of samples changes significantly between iterations, and using a fixed score leads to an inconsistency in the optimization objectives, making training less stable~\cite{curriculum-DIH-bilmes-2020}. Also, we avoid techniques like~\cite{curriculum-humna-score-2015} which requires human annotators, as it is not practical in a privacy-preserving framework.


\noindent \textbf{The pacing function:} The pacing function $g_{\lambda}(t)$ determines scaling of the difficulty of the training data introduced to the model at each of the training steps $t$ and it selects the highest scoring samples for training at each step. The pacing function is parameterized by $\lambda = (a, b)$  where $a$ is the fraction of the training budget needed for the pacing function to begin sampling from the full training set and $b$ represents the fraction of the training set the pacing function exposes to the model at the start of training. In this paper, the full training set size and the total number of training steps (budget) are denoted by $N$ and $T$, respectively. Further, we consider five pacing function families, including exponential, step, linear, quadratic, and root (sqrt). The expressions we used for the pacing functions are shown in Table~\ref{tab:pacing-formula} of the SM.  we follow~\cite{curriculum-power-of-2019} in defining the notion of pacing function and use it to schedule how examples are introduced to the training procedure.

\noindent \textbf{The order:} Curriculum learning orders sample from the highest score (easy ones) to lowest score, anti-curriculum learning orders from lowest score to highest, and finally, random curriculum randomly samples data in each step regardless of their scores.





\begin{algorithm}
\caption{The Curriculum FL Framework }\label{alg:fedpeft}

\small
\begin{algorithmic}
\item \hspace{-6mm}
\noindent \colorbox[rgb]{1, 0.95, 1}{
\begin{minipage}{0.9\columnwidth}

\textbf{Input:} $M$ clients indexed by $m$, sampling rate $R\in(0,1]$, participating-client number $K$, communication rounds $R_C$, server model $f$ with $\theta_g$, pacing function $g_\lambda: [T] \to [N]$, scoring function $s: [N] \to \mathbb{R}$, order $o \in \{ "curriculum", "anti", "random"\}$,

\end{minipage}
}
\item \hspace{-6mm}
\colorbox[gray]{0.95}{
\begin{minipage}{0.9\columnwidth}
\item  \textbf{Server executes:}

\item     \hspace*{\algorithmicindent} initialize $f$ with $\theta$

\item     \hspace*{\algorithmicindent} \textbf{for } each round $t=0,1,2,...$ \textbf{do}

\item     \hspace*{\algorithmicindent} \quad $\mathbb{S}_t \leftarrow$ (random set of $K$ clients)

\item     \hspace*{\algorithmicindent} \quad \textbf{for} each client $m\in \mathbb{S}_t$ \textbf{in parallel do}

\item     \hspace*{\algorithmicindent} \quad \quad  broadcast $\theta^{t}_g $ to clients

\item     \hspace*{\algorithmicindent} \quad \quad  $\theta_m^{(t)} \leftarrow \text{ClientUpdate}(m,\theta^{(t)}_g)$
\item     \hspace*{\algorithmicindent} \quad \quad $ \theta^{(t+1)}_g= \sum_{m=1}^K \frac{|\mathcal{D}_m|}{\sum_{i=1}^K|\mathcal{D}_i|}\theta_m^{(t)}$ \Comment $\mathcal{D}_m$ is the set of the local data on the client with index $m$.
\item     \hspace*{\algorithmicindent} \textbf{return} $\theta_g^{t+1}$
\end{minipage}
}
\item \hspace{-6mm}
\colorbox[rgb]{0.95, 0.98, 1}{
\begin{minipage}{0.9\columnwidth}

\item  \textbf{ClientUpdate ($m, \theta^t_g$):}


\item     \hspace*{\algorithmicindent} Obtain the score of each data sample using $\theta^{t}_g $ and/or $\theta^{t}_m $ as described in section~\ref{effect-scoring-function}

\item     \hspace*{\algorithmicindent} $(\bf{x_1}, \bf{x_2},..., \bf{x_n}) \leftarrow \rm{sort}(\{{\bf{x_2},..., \bf{x_n}} \}, s, o)$

\item     \hspace*{\algorithmicindent} \textbf{for } $t=1,2,...,T$ \textbf{do}

\item     \hspace*{\algorithmicindent} \quad $\theta_m^t \leftarrow \rm{train}(\theta_g^t, \{{\bf{x_1}, \bf{x_2},..., \bf{x}}_{g(t)} \})$

\end{minipage}
}
\end{algorithmic}
\end{algorithm}

\vspace{-2mm}
\section{Experiment}
\label{sec:exp}
\noindent \textbf{Experimental setting.} To ensure that our observations are robust to the choice of architectures, and datasets, we report empirical results for LeNet-5~\cite{lecun1989backpropagation} architecture on CIFAR-10~\cite{krizhevsky2009learning} and Fashion MNIST (FMNIST)~\cite{xiao2017fashion}, and ResNet-9~\cite{he2016deep} architecture for
CIFAR-100~\cite{krizhevsky2009learning} datasets. All models were trained using SGD with momentum. Details of the implementations, architectures, and hyperparameters can be found in Section~\ref{impelement-detail} of the SM.



\noindent \textbf{Baselines and Implementation}. To provide a comprehensive study on the efficacy of CL on FL setups, we consider the predominant approaches to FL that train a global model, including FedAvg~\cite{mcmahan2017communication}, ~FedProx~\cite{fedprox-smith-2020},~SCAFFOLD~\cite{scaffold-2020}, and FedNova~\cite{FedNova-2020}. In all experiments, we assume 100 clients are available, and 10$\%$ of them are sampled randomly at each round. Unless stated otherwise, throughout the experiments, the number of communication rounds is 100, each client performs 10 local epochs with a batch size of 10, and the local optimizer is SGD. To better understand the mutual impact of different data partitioning methods in FL and CL, we consider both federated heterogeneous (Non-IID) and homogeneous (IID) settings. In each dataset other than IID data partitioning settings, we consider two different federated heterogeneity settings as in~\cite{vahidian-pacfl-2022, Vahidian-FLIS-2022}: Non-IID label skew ($20\%$), and Non-IID Dir$(\beta)$.

\subsection{Effect of scoring function in IID and Non-IID FL}
\label{effect-scoring-function}
In this section, we investigate five scoring functions. As discussed earlier, in standard centralized training, samples are scored using the loss value of an expert pre-trained model. Given a pre-trained model $f_{\theta}:\mathcal{X}\to \mathcal{Y}$, the score is defined as $ s_i(x_i)=\frac{r_i}{\sum_{i} r_i}$, where $r_i=\frac{1}{{\mathcal{L}}(y_i, f_{\theta}(x_i))}$, with ${\mathcal{L}}$ being the inference loss. In this setup, a higher score corresponds to an easier sample.

In FL~\cite{mcmahan2017communication, Vahidian-federated-2021}, a trusted server broadcasts a single initial global model, $\theta_g$, to a random subset of selected clients in each round. The model is optimized in a decentralized fashion by multiple clients who perform some steps of SGD updates $(\theta_k=\theta_g-\eta \nabla \mathcal{L}_k)$. The resulting model is then sent back to the server, ultimately converging to a joint representative model. With that in mind, in our setting, the scores can be obtained by clients either via the global model that they receive from the server, which we name as $s_G$ or by their own updated local model, named as $s_L$ or the score can be determined based on the average of the local and global model loss, named as $s_{LG}$\footnote{Since it produces very similar results to $s_G$, we skipped it.}.

We further consider another family of scoring that is based on ground truth class prediction. In particular, in each round, clients receive the global model from the server and get the prediction using the received global model and the current local model as $\hat y_G$ and $\hat y_L$, respectively. For those samples whose $\hat y_L$ and $\hat y_G$ do not match, the client tags them as hard samples and otherwise as easy ones. This scoring method is called $s_{LG}^{pred}$. Further, ground truth class prediction and scoring can be solely done by the global model or the client's local model, which end up with two other different scoring methods, namely  $s_{G}^{pred}$, and $s_{L}^{pred}$ respectively. This procedure is described in Algorithm~\ref{alg:fedpeft}.

\begin{figure}[t]

\centering
\begin{subfigure}{0.48\linewidth}
    \includegraphics[width=1\linewidth]{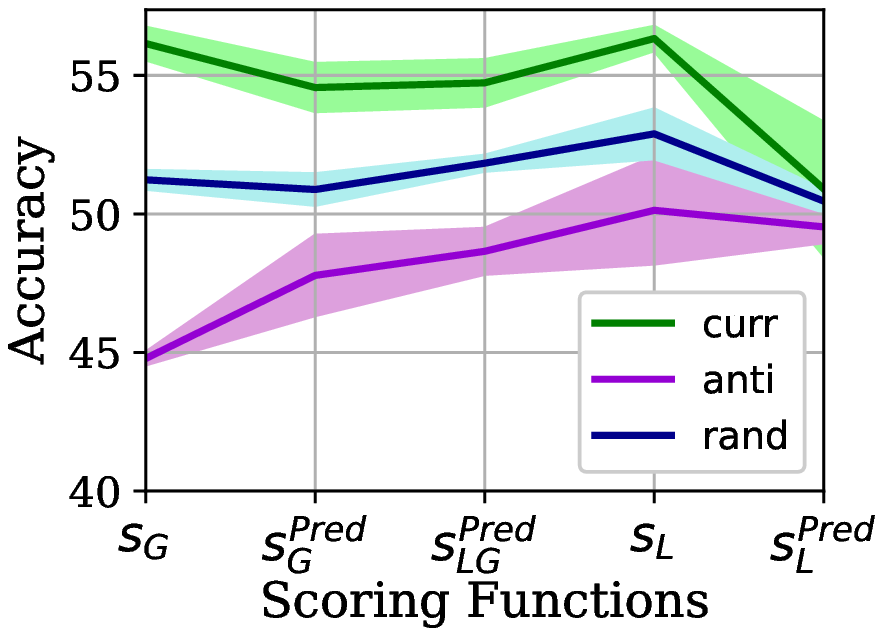}
\end{subfigure}
\begin{subfigure}{0.48\linewidth}
    \includegraphics[width=1\linewidth]{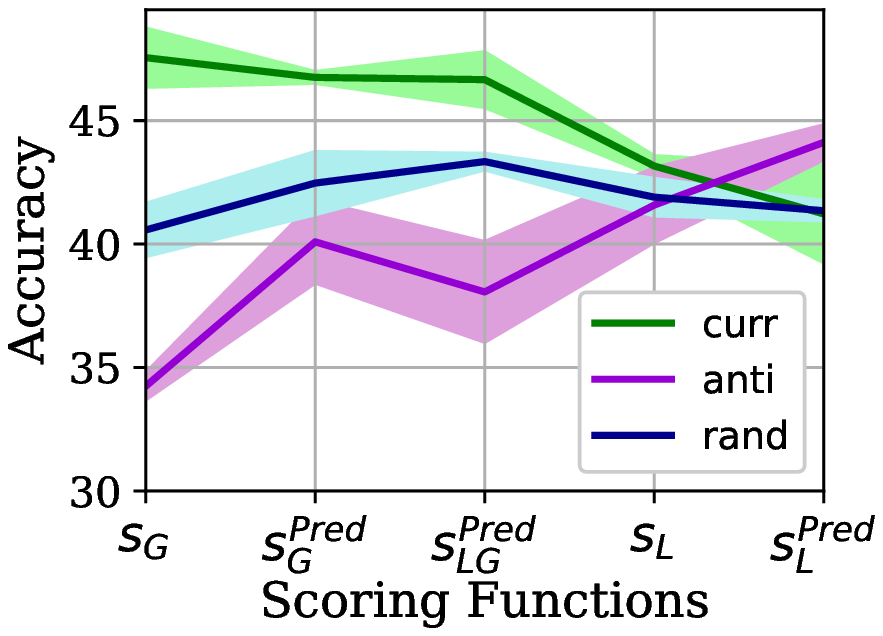}
\end{subfigure}
\vspace{-4mm}
\caption{\footnotesize \textbf{Scoring client samples based on the global model ($s_G$) provides the most accurate scores for all levels of Non-IIDness. Scoring based on the local model ($s_L^{pred}$) provides the least accurate scores, especially when data are Non-IID, as it provides the worst accuracy.} Evaluating the effect of using different scoring methods on accuracy when the clients employ curriculum, anti-curriculum, or random ordering during their local training on CIFAR-10 with IID data (left) and Non-IID (2) (right). All curricula use the linear pacing function with $a = 0.8$ and $b = 0.2$. We run each experiment three times for 100 communication rounds with 10 local epochs and report the mean and standard deviation (std) for global test accuracy. Note that the results for vanilla FedAvg for the left figure, and the right one are $52.30 \pm 0.86$, and  $41.96 \pm 1.77$, respectively.}
\vspace{-4mm}
 \label{fig:scoring-func}
 \vspace{-2mm}
\end{figure}

Fig.~\ref{fig:scoring-func} demonstrates what the impact of using these various scoring methods is on the global accuracy when curriculum, anti-curriculum, or random ordering is exploited by the clients in the order in which their CIFAR-10 examples are learned with FedAvg under IID and Non-IID (2) data partitions. The results are obtained by averaging over three randomly initialized trials and calculating the standard deviation (std).

The results reveal that \textbf{first,}\emph{ the scoring functions are producing broadly consistent results except for $s_{L}^{pred}$ for both IID and Non-IID and $s_{L}$ for Non-IID settings. $s_G$ provides the most accurate scores, thereby improving the accuracy by a noticeable margin compared to its peers.} This is quite expected, as the global model is more mature compared to the client model,
\textbf{second, }\emph{ the curriculum learning improves the accuracy results consistently across different scoring functions,} \textbf{third,} \emph{curriculum learning is more effective when the clients underlying data distributions are Non-IID}. To ensure that the latter does not occur by chance, we will delve into this point in detail in subsection~\ref{effect-of-hetero}.  Due to the superiority of $s_G$ relative to others, we set the scoring function to be $s_G$ henceforth. We will further elaborate on the precision of $s_G$ compared to an expert model in Section~\ref{expert-G-compare}.

\subsection{Effect of pacing function and its parameters in IID and Non-IID FL} \label{Effect-of-pacing-function}

In order to study the effect of different families of pacing functions along with the hyperparameters $\lambda=(a,b)$, we test the exponential, step, linear, quadratic, and root function families. We further first fix $b$ to $0.2$ and let $a \in \{ 0.1, 0.5, 0.8\}$\footnote{Note that $b \in[0,1]$. Also,  $a = 0$ or $b = 1$ is equivalent to no ordered training, i.e., standard training.}. The accuracy results are presented in Fig.~\ref{pacing-a-cifar10-iid-ordering-1}. It is noteworthy that the complement of this figure for Non-IID is presented in Fig.~\ref{pacing-a-cifar10-iid-ordering} in the SM. As is evident, for all pacing function families, the trends between the curriculum and the other orderings, i.e., (anti, random)-curriculum are markedly opposite in how they improve/degrade by sweeping $a$ from small values to large ones. The pattern for Non-IID which presented in the SM is almost similar to that of IID. Values of $a \in [0.5, 0.8]$ produce the best results. As can be seen from Fig.~\ref{pacing-a-cifar10-iid-ordering-1}, the best accuracy achieved by curriculum learning outperforms the best accuracy obtained by other orderings by a large margin. For example, in the ``linear" pacing function, the best accuracy achieved for curriculum learning when $a=0.8$ is $56.60 \pm 0.91$  which improved the vanilla results by $4\%$ while that of random when $a=0.1$ is $52.73 \pm 0.81$ and improved vanilla by $0.5\%$. Henceforth, we set $a=0.8$ and the pacing function to linear. After selecting the pacing function and $a$ the final step is to fix these two and see the impact of $b$. Now we let all curricula use the linear pacing functions with a = 0.8 and only sweep $b \in \{0.0025, 0.1, 0.2, 0.5, 0.8\}$ and report the results in Fig.~\ref{fig:cifar10-pacing-b}. Perhaps most striking is that curriculum learning tends to have smaller values of $b$ to improve accuracy, which is in contrast with (random-/anti) orderings. The performance of anti-curriculum shows a significant dependence on the value of $b$. Further, curriculum learning provides a robust benefit for different values of $b$ and it beats the vanilla FedAvg by $4-7\%$ depending upon the distribution of the data.  Henceforth, we fix $b$ to $0.2$.





\begin{figure}[t]
    \includegraphics[width=1\linewidth]{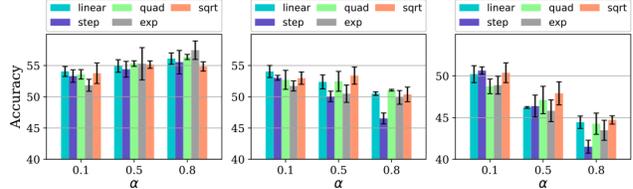}  
    \vspace{-8mm}
\caption{\footnotesize \textbf{ Bigger $a$ values provide better accuracy performance for all pacing function families on IID settings for curriculum learning. But a notable contrast can be seen with random-/anti ordering.} The effect of using different pacing function families and their hyperparameter $a$ on accuracy when the clients employ curriculum, anti-curriculum or random ordering during their local training on CIFAR-10 with IID data. We run each experiment three times for 100 communication rounds with 10 local epochs and report the mean and std for global test accuracy. The figures from left to right are for curriculum, random, and anti ones.}
\label{pacing-a-cifar10-iid-ordering-1}
\vspace{-3mm}
\end{figure}

\begin{figure}[t]
\centering
\begin{subfigure}{0.48\linewidth}
    \includegraphics[width=.97\linewidth]{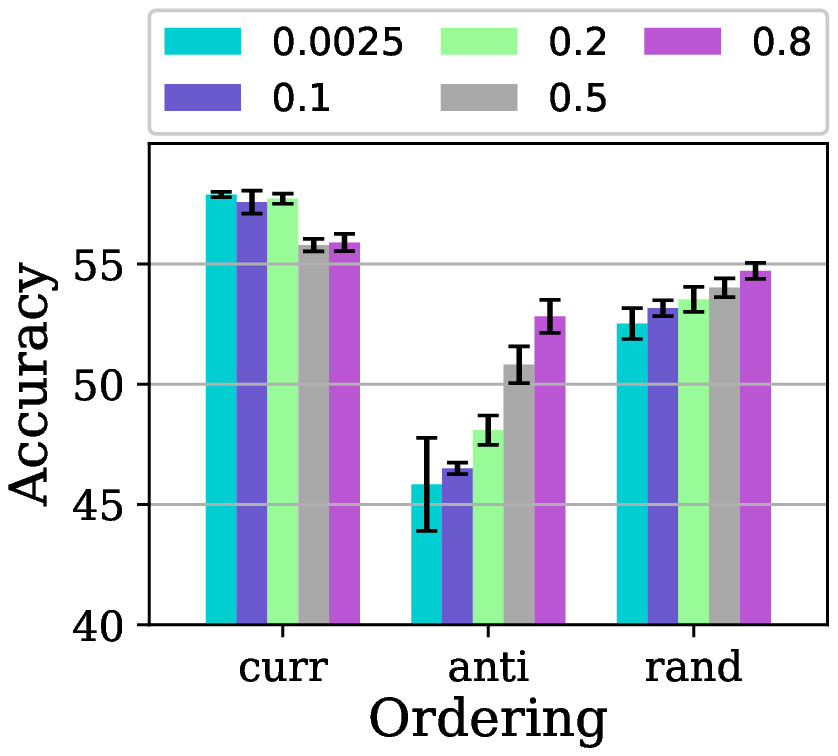}  
\end{subfigure}
\begin{subfigure}{0.48\linewidth}
    \includegraphics[width=.97\linewidth]{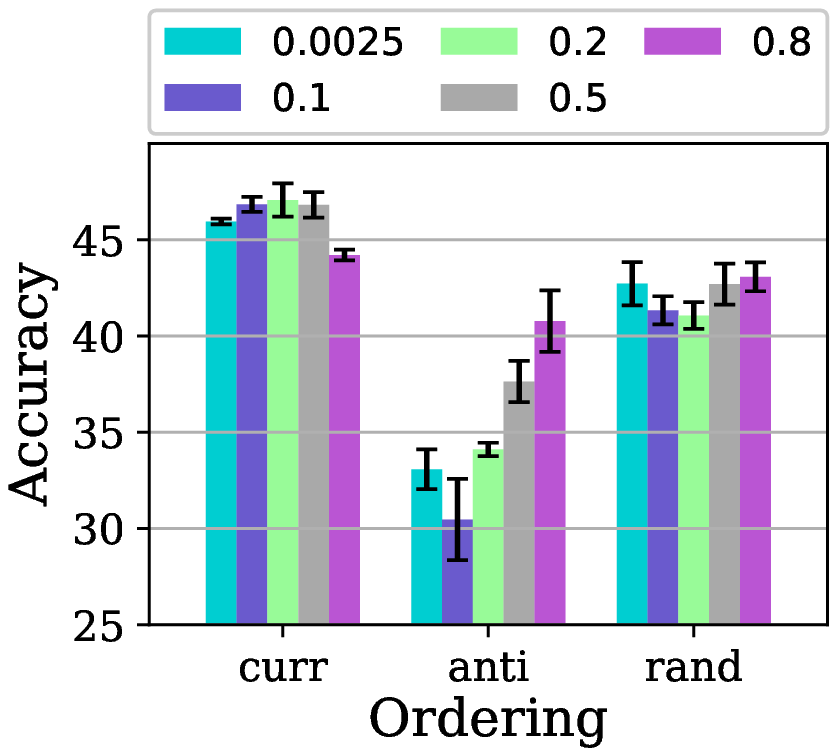}
\end{subfigure}
\vspace{-4mm}
\caption{\footnotesize \textbf{ Smaller $b$ values provide better accuracy performance for both IID and Non-IID settings as further corroborate the benefit of employing curriculum learning.}  Evaluating the effect of hyperparameter $b$ on accuracy when the clients employ curriculum, anti-curriculum, or random ordering
during their local training on CIFAR-10 with IID for FedAVg (left), and with Dir(0.05) for FedAvg (right). All curricula use the
linear pacing functions with a = 0.8. We run each experiment three times for 100 communication rounds with 10 local epochs and report the mean and std for global test accuracy.}
\label{fig:cifar10-pacing-b}
\vspace{-6mm}
\end{figure}

\subsection{Effect of level of data heterogeneity}\label{effect-of-hetero}
Equipped with the ingredients explained in the preceding section, we are now in a position to investigate the significant benefits of employing curriculum learning in FL when the data distribution environment is more heterogeneous. To ensure a reliable finding, we investigate the impact of heterogeneity in four baselines through extensive experiments on benchmark datasets, i.e., CIFAR-10, CIFAR-100, and FMNIST. In particular, we distribute the data between clients according to Non-IID Dir$(\beta)$ defined in~\cite{li2021federated}. In Dir$(\beta)$, heterogeneity can be controlled by the parameter $\beta$ of the Dirichlet distribution. Specifically, when $\beta \to \infty$ the clients' data partitions become IID, and when $\beta \to 0$ they become extremely Non-IID.

To understand the relationship between the ordering-based learning and the level of statistical data heterogeneity, we ran all baselines for different Dirichlet distribution $\beta$ values $\beta \in \{0.05, 0.2, 0.9 \}$.  The accuracy results of different baselines on CIFAR-10 while employing (anti-) curriculum, or random learning with linear pacing functions $(0.8, 0.2)$ and using $s_G$ are presented in Tables~\ref{tab:noniidness-fedavg-cifar10}, \ref{tab:noniidness-fedprox-cifar10}, \ref{tab:noniidness-fednova-cifar10}, and \ref{tab:noniidness-scaffold-cifar10}. For the comprehensiveness of the study, we will present results for CIFAR-100 respectively in Section~\ref{cifar-100-effect-of-level-hetero} of the SM.


The results are surprising: \textbf{The benefits of ordered learning are more prominent with increased data heterogeneity}. The greater the distribution discrepancy between clients, the greater the benefit to curriculum learning. 

If we consider client heterogeneity as distributional skew\cite{zhu2021federated}, then this is logical: easier data samples are those with overall lower variance, both unbiased and skew from the mean, and thus the total collection of CL-easier data samples in a dataset is more IID than the alternative. Thus, in the crucial early phases of training, the training behaves closer to FedAvg/FedProx/SCAFFOLD/FedNova under IID distributions.
Therefore, \emph{CL can alleviate the drastic accuracy drop when clients' decentralized data are statistically heterogeneous, which comes from stable training from IID samples to Non-IID ones, fundamentally improving the accuracy.} This is formalized with quantitative convergence rates in the Section~\ref{analysis-main-paper} of the main paper and in Section~\ref{sec:analysis} of the SM. 

\begin{table} 

\caption{\footnotesize \textbf{Curriculum-learning helps more when training with more severe data heterogeneity across clients}. Understanding the benefit of ordered learning  with increasing  data heterogeneity ($\beta=0.9 \to 0.05$) when clients are trained on CIFAR-10 with \texttt{FedAvg} method.}
\vspace{-3mm}
\centering
\resizebox{\columnwidth}{!}{
\begin{tabular}{lllllllll}
\toprule
Non-IIDness & Curriculum & Anti & Random  & Vanilla \\
\midrule
  

  



Dir($\beta = 0.05$) & $46.34  \pm 1.55$ & $31.16 \pm 3.16$ & $41.91 \pm 2.23$ & $39.56 \pm 4.91$ \\  
  
\midrule
    Dir($\beta = 0.2$)    & $51.09 \pm 0.39$ & $42.34 \pm 1.48$ & $46.35 \pm 1.44$ & $46.75 \pm 0.72$ \\
\midrule

  Dir($\beta = 0.9$)    & $55.36 \pm 0.69$ & $46.86 \pm 0.35$ & $52.42 \pm 0.90$ & $52.19 \pm 0.73$ \\

\bottomrule
\end{tabular}
}
\label{tab:noniidness-fedavg-cifar10}
\vspace{-2mm}
\end{table}

\begin{table} 
\caption{\footnotesize \textbf{Curriculum-learning helps more when training with more severe data heterogeneity across clients}. Understanding the benefit of ordered learning  with increasing data heterogeneity  ($\beta=0.9 \to 0.05$) when clients are trained on CIFAR-10 with \texttt{Fedprox} method.}
\vspace{-3mm}
\centering
\resizebox{\columnwidth}{!}{
\begin{tabular}{lllllllll}
\toprule
 Non-IIDness & Curriculum & Anti & Random  & Vanilla \\
\midrule
 Dir($\beta = 0.05$)   & $47.94 \pm 0.96$ & $36.08 \pm 1.52$ & $42.62 \pm 0.35$ & $41.48 \pm 0.29$ \\
\midrule

 Dir($\beta = 0.2$)    & $50.02 \pm 0.15$ & $40.92 \pm 0.90$ & $46.41 \pm 1.12$ & $46.18 \pm 0.90$ \\
\midrule

 Dir($\beta = 0.9$)    & $56.48 \pm 0.18$ & $48.37 \pm 0.91$ & $51.69 \pm 0.40$ & $53.07 \pm 1.25$ \\

\bottomrule
\end{tabular}
}
\label{tab:noniidness-fedprox-cifar10}
\vspace{-2mm}
\end{table}

\begin{table}  
\caption{\footnotesize \textbf{Curriculum learning helps more when training with more severe data heterogeneity across clients}. Understanding the benefit of ordered learning  with increasing data heterogeneity ($\beta=0.9 \to 0.05$) when clients are trained on CIFAR-10 with \texttt{FedNova} method.}
\vspace{-3mm}
\centering
\resizebox{\columnwidth}{!}{
\begin{tabular}{lllllllll}
\toprule
 Non-IIDness & Curriculum & Anti & Random  & Vanilla \\
\midrule
Dir($\beta = 0.05$)   & $ 43.73 \pm 0.09 $ & $ 28.31 \pm 1.93 $ & $ 37.81 \pm 3.06 $ & $ 31.97 \pm 0.90 $ \\
\midrule

Dir($\beta = 0.2$)    & $ 47.01 \pm 1.89 $ & $ 36.55 \pm 1.42 $ & $ 44.21 \pm 1.00 $ & $ 41.28 \pm 0.30 $ \\
\midrule

Dir($\beta = 0.9$)    & $ 50.74 \pm 0.19 $ & $ 41.76 \pm 0.90 $ & $ 48.87 \pm 0.88 $ & $ 47.230 \pm 1.80 $ \\


\bottomrule
\end{tabular}
}
\label{tab:noniidness-fednova-cifar10}
\vspace{-2mm}
\end{table}

\begin{table} 
\caption{\footnotesize \textbf{Curriculum-learning helps more when training with more severe data heterogeneity across clients}. Understanding the benefit of ordered learning  with increasing  data heterogeneity ($\beta=0.9 \to 0.05$) when clients are trained on CIFAR-10 with \texttt{SCAFFOLD} method.}
\vspace{-3mm}
\centering
\resizebox{\columnwidth}{!}{
\begin{tabular}{lllllllll}
\toprule
 Non-IIDness & Curriculum & Anti & Random  & Vanilla \\
\midrule
Dir($\beta = 0.05$)   & $45.91 \pm 1.17$ & $21.29 \pm 1.82$ & $38.27 \pm 2.19$ & $41.33 \pm 1.30$ \\
\midrule

Dir($\beta = 0.2$)    & $49.69 \pm 1.81$ & $28.69 \pm 0.60$ & $45.29 \pm 1.93$ & $46.62 \pm 0.58$ \\
\midrule

 Dir($\beta = 0.9$)    & $52.05 \pm 1.14$ & $30.75 \pm 0.79$ & $49.25 \pm 0.76$ & $50.24 \pm 0.57$ \\


\bottomrule
\end{tabular}
}
\label{tab:noniidness-scaffold-cifar10}
\vspace{-3mm}
\end{table}

 \section{Curriculum on Clients} \label{curr-on-client}
 \vspace{-2.7mm}
The technique of ordered learning presented in previous sections is designed to exploit the heterogeneity of data at the clients but is not geared to effectively leverage the \underline{heterogeneity between the clients} that, as we discuss further, naturally emerges in the real world.

In the literature, some recent works have dabbled with the idea of smarter client selection, and many selection criteria have been suggested, such as importance sampling, where the probabilities for clients to be selected are proportional to their importance measured by the norm of update~\cite{Optimal-Client-Sampling-2022}, test accuracy~\cite{client-selection-iot-2021}. The ~\cite{power-of-choice-2020} paper proposes client selection based on local loss where clients with higher loss are preferentially selected to participate in more rounds of federation, which is in stark contrast to~\cite{client-selection-2022-curr} in which the clients with a lower loss are preferentially selected. It’s clear from the literature that the heterogeneous environment in FL can hamper the overall training and convergence speed~\cite{haddadpour2019convergence, mahdavi2020-APFL}, but the empirical observations on client selection criteria are either in conflict or their efficacy is minimal. In this section, inspired by curriculum learning, we try to propose a more sophisticated mechanism of client selection that generalizes the above strategies to the FL setting.

 \subsection{Motivation}
In the real world, the distributed dataset used in FL is often generated in-situ by clients, and the data is measured/generated using a particular sensor belonging to the client. For example, consider the case of a distributed image dataset generated by smartphone users using the onboard camera or a distributed medical imaging dataset generated by hospitals with their in-house imaging systems. In such scenarios, as there is a huge variety in the make and quality of the imaging sensors, the data generated by the clients is uniquely biased by the idiosyncrasies of the sensor, such as the noise, distortions, quality, etc. This introduces variability in the quality of the data at clients, in addition to the Non-IID nature of the data. However, it is interesting to note that these effects apply consistently across all the data at the specific client.

From a curriculum point of view, as the data points are scored and ordered by difficulty, which is just a function of the loss value of that data point, these idiosyncratic distortions uniformly affect the loss/difficulty value of all the data at that particular client. Also, it is possible that the difficulty among the data points at the particular client is fairly consistent as the level of noise, quality of the sensor, etc. are the same across the data points. This bias in difficulty varies between clients but is often constant within the same client. Naturally, this introduces a \underline{heterogeneity in the clients} participating in the federation. In general, this can be thought of as some clients being "easy" and some being "difficult". \emph{ We can quantify this notion of consistency in difficulty by the standard deviation in the score of the data points at the client.}
 




\begin{figure}[t]
    \centering
     \begin{subfigure}[b]{0.3\linewidth}
         \centering
         \includegraphics[width=\textwidth]{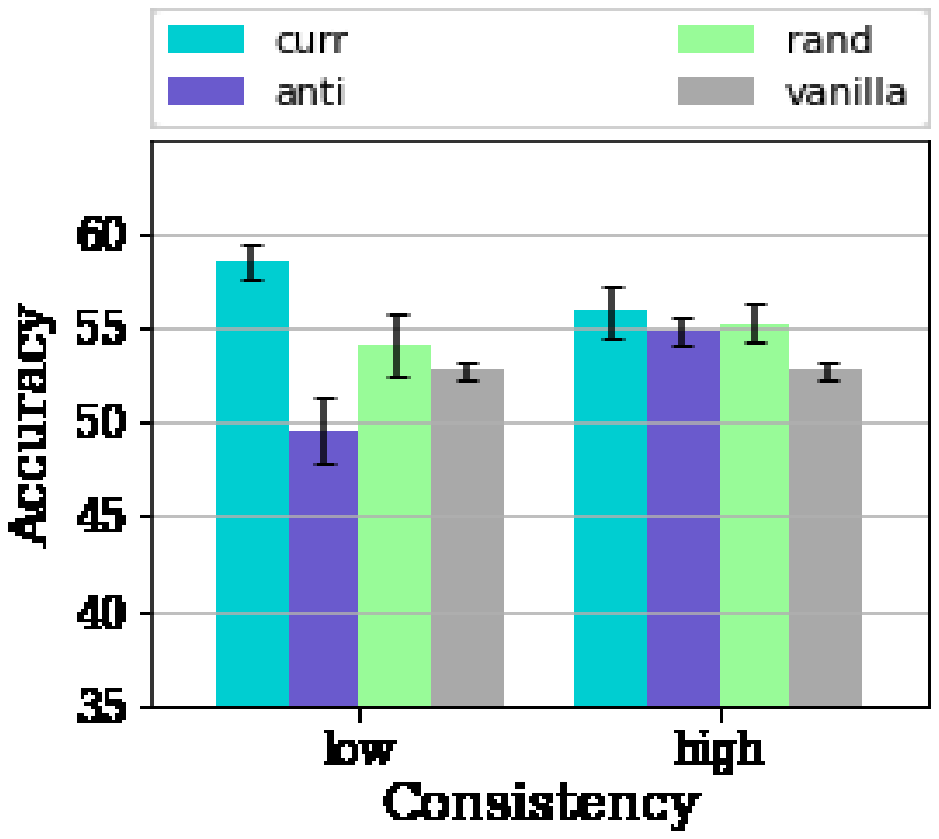}
         \caption{IID}
         \label{subfig:DC_noeffect_iid}
     \end{subfigure}
     \hfill
     \begin{subfigure}[b]{0.3\linewidth}
         \centering
         \includegraphics[width=\textwidth]{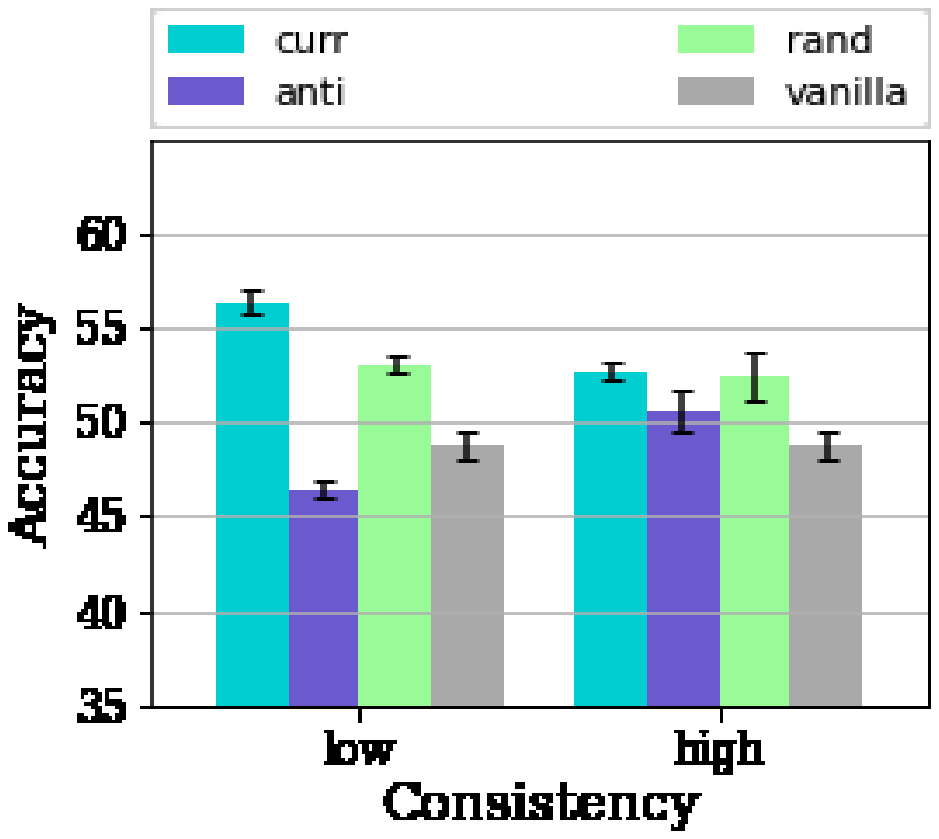}
         \caption{Dir(0.2)}
         \label{subfig:DC_noeffect_dir02}
     \end{subfigure}
     \hfill
     \begin{subfigure}[b]{0.3\linewidth}
         \centering
         \includegraphics[width=\textwidth]{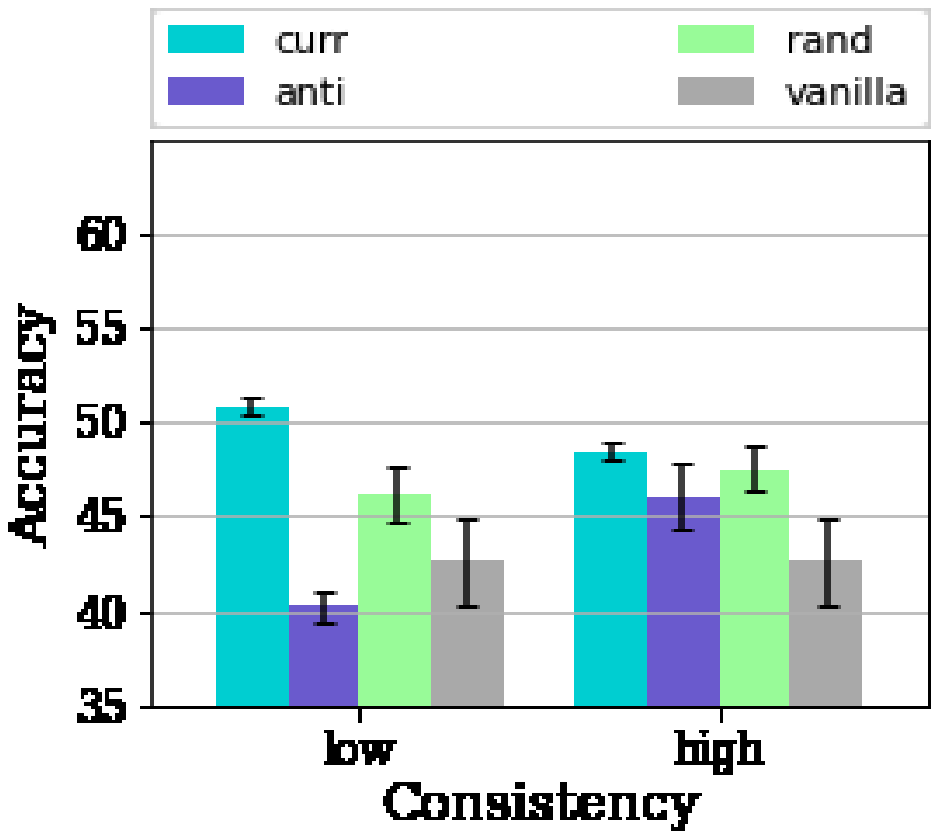}
         \caption{Dir(0.05)}
         \label{subfig:DC_noeffect_dir005}
     \end{subfigure}
    \vspace{-2mm}
\caption{\footnotesize \textbf{Consistency in data difficulty at the client hurts the efficacy of curricula.} The effect of consistency in the difficulty distribution at the client nullifies the effect of curricula. The values plotted are for FedAvg on CIFAR-10. The standard deviation values of (Low, High) consistency for the IID are $(0.52,0.01)$, Dir(0.2) are $(0.51,0.14)$, and Dir(0.05) are $(0.50,0.13)$. Note that we use Algorithm~\ref{alg:partiiton_diff_dist} to construct partitions with varying difficulty, and as detailed in Section~\ref{sec:difficulty_based_partitioning} it is not possible to control the partition difficulty value with arbitrary precision, hence the above minor variations. The Low consistency scenario is generated using $f_{ord}=0.0$ and the high consistency scenario uses $f_{ord}=1.0$.}
\label{fig:DC_noeffect_E}
\end{figure}

When the standard deviation of the intra-client data is low, i.e., when the difficulty of the data within a client is consistent, we find that the curriculum on FL behaves very differently in these kinds of scenarios. We observe the advantage of curriculum diminishes significantly and has similar efficacy as that of random curricula as shown in Fig~\ref{fig:DC_noeffect_E}. The advantage of curriculum can be defined as $A_{o} = accuracy(o) - accuracy(vanilla)$, where $o \in \{curr,anti,rand\}$.

\subsection{Client Curriculum}
\vspace{-1mm}



We propose to extend the ideas of curriculum onto the set of clients, in an attempt to leverage the heterogeneity in the clients. To the best of our knowledge, our paper is the first attempt to introduce the idea of \texttt{curriculum on clients} in an FL setting. Many curricula ideas can be neatly extended to apply to the set of clients. In order to define a curriculum over the clients, we need to define a scoring function, a pacing function, and an ordering over the scores. We define the client loss as the mean loss of the local data points at the client (Eq.~\ref{eq:client_loss}), and the client score can be thought of as inversely proportional to the loss. The ordering is defined over the client scores.\vspace{-2.7mm}

\begin{equation}
    \mathcal{L}_k =\frac{1}{\|\mathbb{D}_m\|}\sum\limits_{j}^{\|\mathbb{D}_m\|} l_j  
\label{eq:client_loss}
\end{equation}

\noindent where $m$ is the index for client, $\mathbb{D}_m$ represents the dataset at client $m$, $l_j$ is the loss of $j$th data point in $\mathbb{D}_m$. 

The pacing function, as in the case of data curricula, is used to pace the number of clients participating in the federation. Starting from a small value, the pacing function gradually increases the number of clients participating in the federation. The action of the pacing function amounts to scaling the participation rate of the clients.

The clients are scored and rank-ordered based on the choice of ordering, then a subset of size $K^{(t)}$ of the $K$ clients is chosen in a rank-ordered fashion. The value $K^{(t)}$ is prescribed by the pacing function. The $K^{(t)}$ clients are randomly batched into mini-batches of clients. These mini-batches of clients subsequently participate in the federation process. Thereby, we have two sets of curricula, one that acts locally on the client data and the other that acts on the set of clients. Henceforth we will refer to these as the \texttt{data curriculum} and the\texttt{ client curriculum}, respectively. We study the interplay of these curricula in Section~\ref{sec:ablation_study}.

\begin{figure}[t]
    \centering
     \begin{subfigure}[b]{0.3\linewidth}
         \centering
         \includegraphics[width=\textwidth]{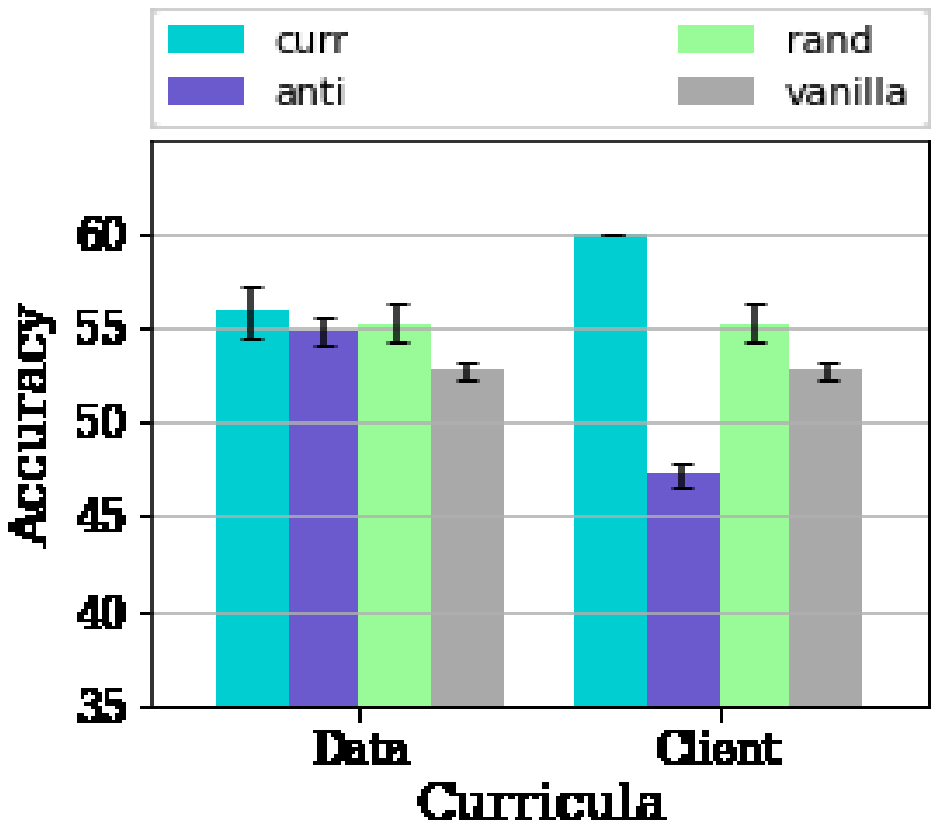}
         \caption{IID}
         \label{subfig:CC_effect_iid}
     \end{subfigure}
     \hfill
     \begin{subfigure}[b]{0.3\linewidth}
         \centering
         \includegraphics[width=\textwidth]{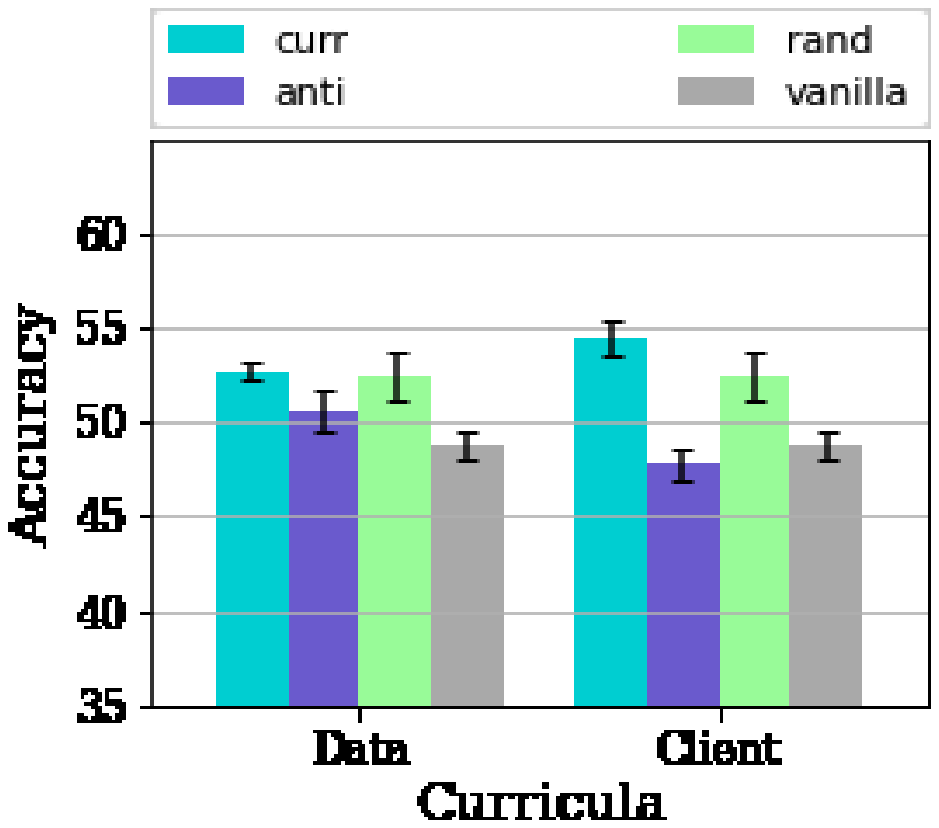}
         \caption{Dir(0.2)}
         \label{subfig:CC_effect_dir02}
     \end{subfigure}
     \hfill
     \begin{subfigure}[b]{0.3\linewidth}
         \centering
         \includegraphics[width=\textwidth]{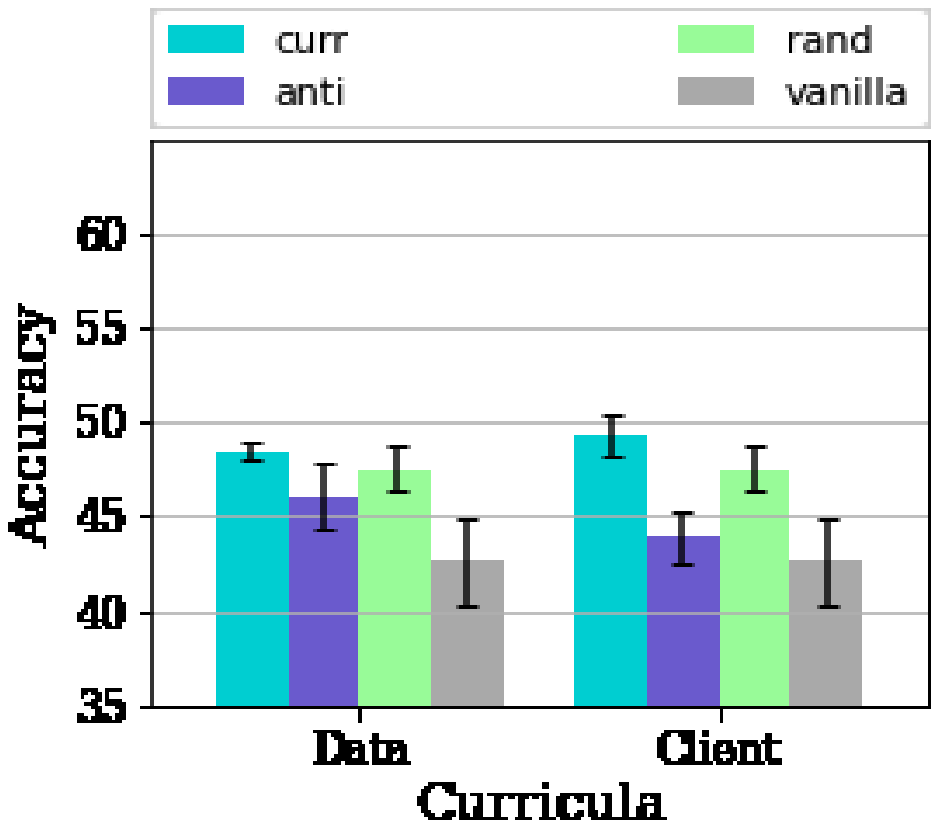}
         \caption{Dir(0.05)}
         \label{subfig:CC_effect_dir005}
     \end{subfigure}
    \vspace{-2mm}
\caption{\footnotesize \textbf{Client curriculum does not suffer from low heterogeneity in the data difficulty and is effective when data curriculum is not.} The scenario shown here is the same as the scenario with high local consistency from Fig~\ref{fig:DC_noeffect_E}. As we observe the client curriculum is able to overcome the limitations of the data curriculum.}
\label{fig:CC_effect_E}
\vspace{-3mm}
\end{figure}
Fig.~\ref{fig:CC_effect_E} confirms that the benefits of the algorithm severely depend on the data of the client having a diverse set of difficulties.  As is evident from Fig.~\ref{fig:CC_effect_E}, we are able to realize an $A_{curr}$ of $5.67-7.19$\% for the different values of Non-IIDness using our proposed client curriculum algorithm in the scenario with high consistency in the client data where the data curriculum has reduced/no benefits. \textit{This illustrates that the client curriculum is able to effectively overcome the limiting constraint of local variability in data and is able to leverage the heterogeneity in the set of clients to its benefit.}

\subsection{Difficulty based partitioning} \label{sec:difficulty_based_partitioning}


For the experiments in this section, we require client datasets ($\mathbb{D}_m$) of varying difficulty at the desired level of Non-IIDness. In order to construct partitions of varying difficulty, we need to address two key challenges: one, we need a way to accurately assess the difficulty of each of the data points, and two, we need to be able to work with different levels of Non-IIDness. To address the first challenge, we rank the data points in order of difficulty using an a priori trained expert model $\theta_{\mathcal{E}}$ that was trained on the entire baseline dataset and has the same network topology as the global model. As the expert model has the same topology as the model that we intend to train and as it is trained on the entire dataset, it is an accurate judge of the difficulty of the data points. Interestingly, this idea can be extended to be used as a scoring method for curriculum as well. We call this scoring method the expert scoring $s_{E}$. We look at this in greater detail in Section~\ref{expert-G-compare}.

To address the second challenge, a possible solution is to first partition the standard dataset into the desired Non-IID partitions using well-known techniques such as the Dirichlet distribution, followed by adding different levels of noise to the samples of the different data partitions. This would partition with varying difficulty; however, doing so would alter the standard baseline dataset, and we would lose the ability to compare the results to known baseline values and between different settings. We would like to be able to compare our performance results with standard baselines, so we require a method that does not alter the data or resort to data augmentation techniques, and we devise a technique that does just that. 

Starting with the baseline dataset, we first divide it into the desired Non-IID partitions the same as before, but then instead of adding noise to the dataset, we attempt to reshuffle the data among partitions in such a way that we create "easy" partitions and "hard" partitions. This can be achieved by ordering the data in increasing order of difficulty and distributing the data among the partitions starting from the "easy" data points, all the while honoring the Non-IID class counts of each of the partitions as determined by the Non-IID distribution. The outline is detailed in Algorithm~\ref{alg:partiiton_diff_dist}. It is noteworthy that, although we are able to generate partitions of varying difficulty, we do not have direct control over the "difficulty" of each of the partitions and hence cannot generate partitions with an arbitrary distribution of difficulty as can be done by adding noise.

\begin{algorithm}[t]
\caption{Partition Difficulty Distribution}\label{alg:partiiton_diff_dist}

\small
\begin{algorithmic}
\item \hspace{-6mm}
\noindent \colorbox[rgb]{1, 0.95, 1}{
\begin{minipage}{0.9\columnwidth}
\textbf{Input:} partitions $\{\mathbb{P}_0,\mathbb{P}_1,...,\mathbb{P}_N\}$ of the input dataset $\mathbb{D}$ of $C$ classes indexed by $c$, fraction of each partition to replace $f_{ord}\in(0,1]$, expert model $\theta_{\mathcal{E}}$

\end{minipage}
}
\item \hspace{-6mm}
\colorbox[gray]{0.95}{
\begin{minipage}{0.9\columnwidth}
\item  \textbf{Class prior of partitions and dataset:}

\item     \hspace*{\algorithmicindent} \textbf{for } each partition $i=0,1,2,...,N$ \textbf{do}
\item     \hspace*{\algorithmicindent} \quad $\mathbb{P}_i \leftarrow$ count(data points of class $c$ in partition $i$)
\item     \hspace*{\algorithmicindent}    
\item     \hspace*{\algorithmicindent} Compute loss ($\mathcal{L}$) for each data point in $\mathbb{D}$ using $\theta_{\mathcal{E}}$
\item     \hspace*{\algorithmicindent} $\mathbb{D}_{c} \leftarrow$ argsort($\mathcal{L}_c$)

\end{minipage}
}
\item \hspace{-6mm}
\colorbox[rgb]{0.95, 0.98, 1}{
\begin{minipage}{0.9\columnwidth}

\item  \textbf{Reconstitute partition:}


\item     \hspace*{\algorithmicindent} $id_c = cumsum_i(N_{i,c})$
\item 
\item     \hspace*{\algorithmicindent} Distribute $f_{ord}$ 
\item     \hspace*{\algorithmicindent} \textbf{for } each partition $i=0,1,2,...,N$ \textbf{do}
\item     \hspace*{\algorithmicindent} \quad $\mathbb{P}_i \leftarrow \mathbb{P}_i \cup$ partition($f_{ord}*N_{i,c}$ elements of $\mathbb{D}_c$ beginning at $id_{i,c}$)
\item     \hspace*{\algorithmicindent} $\mathbb{D'}_c \leftarrow$ remaining elements of $\mathbb{D}_c$
\item 
\item     \hspace*{\algorithmicindent} Distribute remaining $(1-f_{ord})$
\item     \hspace*{\algorithmicindent} \textbf{for } each partition $i=0,1,2,...,N$ \textbf{do}
\item     \hspace*{\algorithmicindent} \quad $\mathbb{P}_i \leftarrow \mathbb{P}_i \cup$ random($(1-f_{ord})*N_{i,c}$ elements of $\mathbb{D'}_c$)

\end{minipage}
}
\end{algorithmic}
\vspace{-1mm}
\end{algorithm}

\vspace{-1mm}
\subsection{Expert guided and self-guided curricula}\label{expert-G-compare}

The scoring method $s_E$, as discussed above, can also be used to guide the learning process in a curriculum learning setting. As the expert model used for scoring shares the same network topology as the global model that we intend to train, and as the expert was trained on the entire dataset, the expert-guided curricula can be thought of as a pedagogical mode of learning.

\begin{figure}[t]
\vspace{-3mm}
    \centering
     \begin{subfigure}[b]{0.48\linewidth}
         \centering
         \includegraphics[width=\textwidth]{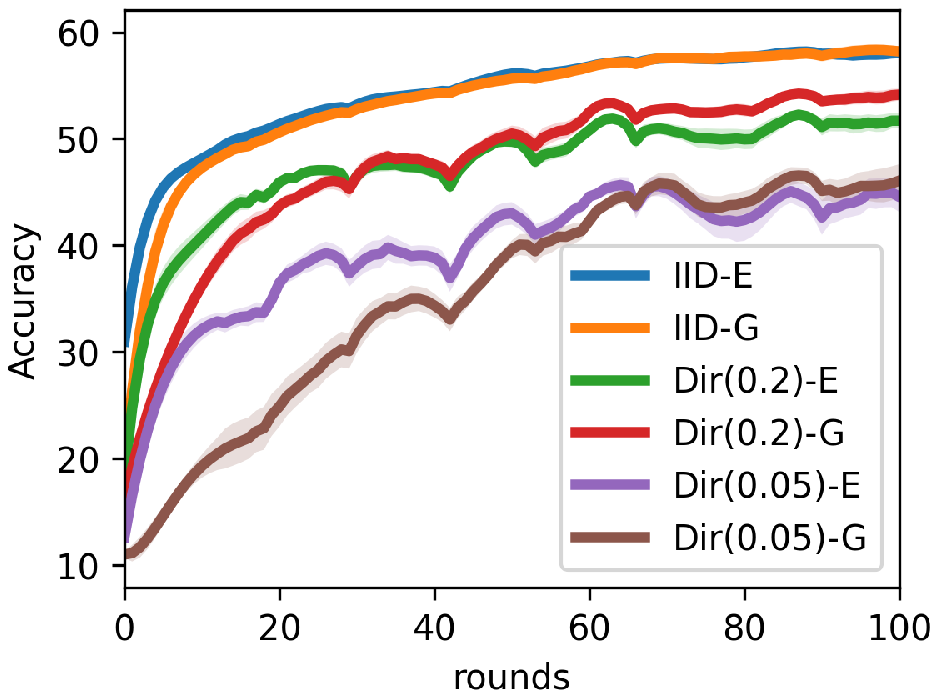}
         \vspace{-6mm}
         \caption{Client curriculum}
         \label{subfig:CC_curr_GisE}
     \end{subfigure}
     \hfill
     \begin{subfigure}[b]{0.48\linewidth}
         \centering
         \includegraphics[width=\textwidth]{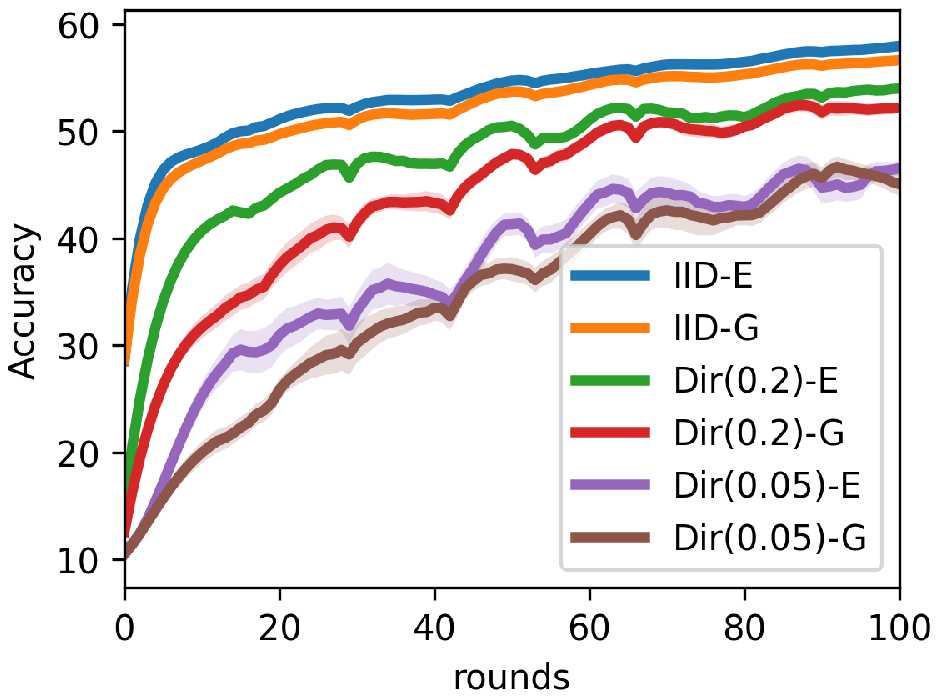}
         \vspace{-6mm}
         \caption{Data curriculum}
         \label{subfig:DC_curr_GisE}
     \end{subfigure}
    \vspace{-4mm}
\caption{\footnotesize \textbf{Effect of expert scoring $s_{E}$ and $s_{G}$ on "curr" curriculum.} Plotted here is the evolution of the global model's accuracy over the course of federation for $\beta \in \{0.05,0.2\}$ and IID with an ordering of 'curr', using FedAvg. $s_{G}$ and $s_{E}$ scoring functions have similar behavior on the Client Curricula (Left) and Data Curricula (Right).}
\label{fig:GisE_curr}
\vspace{-5mm}
\end{figure}



The global model accuracy at different rounds of federation is depicted in Fig.~\ref{fig:GisE_curr}. We see a clear trend in Fig~\ref{fig:GisE_curr} that $s_E$ outperforms $s_G$ in the initial rounds, but $s_G$ converges to $s_E$ over the rounds. Also, $s_{G}$ accuracy in the initial rounds very closely approximates the random scoring accuracy. The $s_{G}$ scoring method is a self-guided curriculum, that uses the global model. The global model is just random (noisy) in the initial rounds of federation, and hence the curricula it produces are also random, thereby closely approximating the performance of the random curriculum. As the global model is refined over time, it becomes better at determining the "true" curricula, eventually converging on the $s_E$ curve. The model trained with $s_E$ benefits from curriculum effects from the first round and thus starts strong.




\subsection{Ablation study} \label{sec:ablation_study}
\vspace{-2mm}

\begin{figure}[t]
    \centering
    \includegraphics[width=0.45\linewidth]{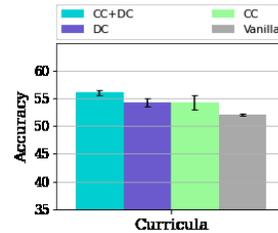}
    \vspace{-5mm}
    \caption{\footnotesize \textbf{Synergic Effect of Client and Data curricula.} CC here refers to Client Curriculum and DC refers to Data Curriculum. The figure shows the synergic effects of the curricula.}
    \label{fig:CC_w_DC_accuracy}
     \vspace{-5mm}
\end{figure}

In this section, we show the interplay between the \texttt{data curriculum} and the \texttt{client curriculum} and measure their contributions towards the global model's accuracy. As reported in Fig~\ref{fig:CC_w_DC_accuracy}, we observe that the client curriculum and the data curriculum independently outperform the baseline by about $2-3$\%, and \underline{we observe a synergic effect of the combination} that outperforms both the curricula and the baseline by about $5$\%. 


\vspace{-2mm}
\section{Theoretical Analysis and Convergence Guarantees}\label{analysis-main-paper}
\vspace{-2mm}
Now we attempt to analytically motivate the improved performance of CL in general, and for heterogeneous data in particular.

\noindent \textbf{Convergence Rate Advantages of Curriculum Learning}
Consider a standard loss function of the least squares form,\vspace{-3mm}
\begin{equation}
    \mathcal{L}(\theta,\{x_i,y_i\}) =\frac{1}{2N}\sum\limits_{i=1}^N (f(\theta,x_i)-y_i)^2 \notag
\end{equation}
Compute the generic form of the Hessian,\vspace{-3mm}
\begin{align}
\nabla^2_{\theta\theta} \mathcal{L}(\theta,\{x_i,y_i\}) =&\frac{1}{N}\sum\limits_{i=1}^N \left(\nabla_{\theta} f(\theta,x_i)\nabla_{\theta} f(\theta,x_i)^T\right. \nonumber\\ & \left.+\nabla^2_{\theta\theta} f(\theta,x_i) (f(\theta,x_i)-y_i)\right) \nonumber
\end{align}

Note that the Fisher information matrix, or Gauss-Newton term $\nabla_{\theta} f(\theta,x_i)\nabla_{\theta} f(\theta,x_i)^T$ is expected to be positive definite and independent of each samples loss value, however, the greater the magnitude of the overall loss $(f(\theta,x_i)-y_i)$ the greater the potential influence of the Hessian of the neural network model $\nabla^2_{\theta\theta} f(\theta,x_i)$ on the overall Hessian of the objective function. Thus, inherently, curriculum training makes the initial objective function more convex than otherwise. This has the clear consequence of enabling faster optimization trajectories at the beginning of the training process.

\noindent \textbf{Distribution Skew and Heterogeneous Data.} Formally, in terms of the optimization landscape and criteria, the presence of Non-IID data is often modeled in terms of the quantitative features of the appropriate model in a purely deterministic sense, a different minimizer, etc. Distributionally, however, one can observe, see e.g.~\cite{zhu2021federated}, that data discrepancy across clients is often manifested as \emph{skew}. Skew is the third moment of a random variable that indicates that there is a preferential direction in the uncertainty. As an example, image data is distributed across clients by giving the left division of an image—e.g., the left face of a cat—to one client and the right to another. See Fig.~\ref{fig:skewcurricula} for an illustration.

\begin{figure}
    \centering
    \includegraphics[scale=0.5]{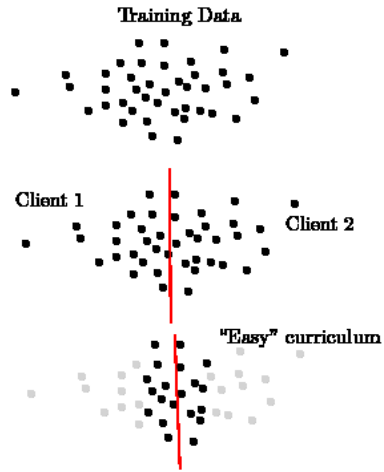}
        \vspace{-2mm}
    \caption{Illustration of skew-based heterogeneous data distribution across clients and curriculum learning mitigation thereof.}
    \label{fig:skewcurricula}
    \vspace{-6mm}
\end{figure}

A simple and transparent way to model this is to use the biased SGD framework. Specifically, there is an underlying objective function of interest $f(x)$, however, each client only has access to a biased stochastic gradient of this function. Uniquely in the case we consider and model, the bias adds up to zero across clients. We shall use the notation of~\cite{ajalloeian2020convergence} although for completeness we acknowledge the predecessor~\cite{karimi2019non}. To the best of our knowledge, we present the first analysis of federated averaging with heterogeneous data using the biased SGD framework, despite how naturally it models the training procedure given standard distributional patterns in splitting training data across clients. In the SM, we develop this model formally and provide quantitative convergence results. Informally, the overall findings can be summarized as follows:
\begin{enumerate}
    \item For strongly convex objectives, the bias introduces error to the asymptotic distance to the optimal solution, the amount of which can be decreased by appropriate annealing stepsizes according to the client or data-based CL.
    
    \item For nonconvex objectives with a bounded gradient assumption, it appears that with a sufficiently annealed stepsize, standard centralized sublinear ergodic rates of convergence to zero approximate stationarity in expectation can be recovered.
\end{enumerate}

\vspace{-2mm}
\section{Conclusion}\label{Conclusion}
\vspace{-2mm}
In this work, we provided a comprehensive study on the benefit of employing CL in FL under both homogeneous and heterogeneous setting. We further ran extensive experiments on a broad range of curricula and pacing functions over three datasets, CIFAR10, CIFAR100, and FMNIST and demonstrated that ordered learning can have noticeable benefits in federated training. Surprisingly, we found empirically that CL can be more beneﬁcial when the clients underlying data distributions are significantly Non-IID. By studying the convergence behavior of FL using a novel biased SGD model based on the observation of data heterogeneity as distributional skew, we were able to theoretically explain this phenomenon. Moreover, we proposed curriculum on clients for the first time. Our results show that the order in which clients are participated in the federation plays an important role in the accuracy performance of the global model. In particular, training the global model in a meaningful order, from the easy clients to the hard ones, using curriculum learning can provide performance improvements over the random sampling approach.\\






\setcounter{section}{0}
{\Large \textbf{Supplementary Document}}
\normalsize

The supplementary material is organized as follows: Section~\ref{prelim} presents preliminaries; Section~\ref{sec:analysis} provides the convergence of theory; Section \ref{Effect-of-amount-of-data} studies the effect of the amount of data that each client owns on its benefit from CL; in Section \ref{Effect-of-pacing-function-sup} additional experiments are provided to evaluate the effect of pacing function and its parameters in IID and non-IID FL; Section~\ref{cifar-100-effect-of-level-hetero} studies the correlation between the ordering based learning and the level of statistical data heterogeneity on CIFAR-100; 
Section~\ref{sec:related-work} presents the related work to this paper; 
Section~\ref{impelement-detail} contains implementation details; and finally, Section~\ref{Conclusion} concludes the paper.


\section{Preliminary} \label{prelim}
The five function families used throughout, including exponential, step, linear, quadratic, and root, and their expressions can be seen in Table~\ref{tab:pacing-formula} and Fig.~\ref{fig:pacing-families}.

\begin{table}[H]  
\caption{\footnotesize The five families of pacing functions we employed in this paper. The parameter $a$ determines the fraction of training time until all data is used. Parameter $b$ sets the initial fraction of the data used.}
\vspace{-3mm}
\centering
\resizebox{0.6\columnwidth}{!}{
\begin{tabular}{ll}
\toprule
Pacing Function & Expression  \\
\midrule
Exponential     & $Nb + \frac{N(1-b)}{e^{10}-1} (e^{\frac{10t}{aT}}-1)$  \\
Step            & $Nb + N{\lfloor \frac{t}{aT} \rfloor}$  \\
Root (Sqrt)     & $Nb + \frac{N-b}{\sqrt{aT}} \sqrt{t}$  \\
Linear          & $Nb + \frac{N-b}{aT} t$  \\
Quadratic       & $Nb + \frac{N-b}{{(aT)}^2} t^2$  \\
\bottomrule
\end{tabular}
}
\label{tab:pacing-formula}
\end{table}

\begin{figure}[H]
    \includegraphics[width=1\linewidth]{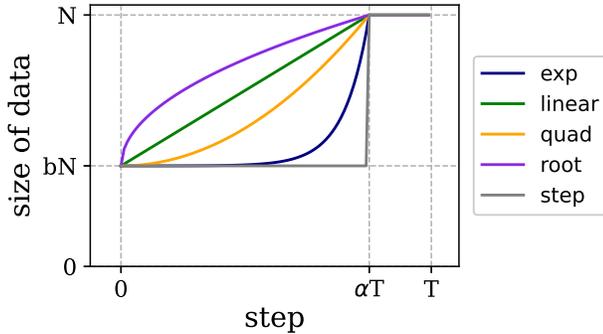}  
    \vspace{-8mm}
\caption{\footnotesize Pacing function curves of different families are used throughout the paper. As shown, the hyperparameter $\alpha$ specifies the fraction of the training step until all data is used in training. The hyperparameter $b$ determines the initial fraction of the data used in training.}
\label{fig:pacing-families}
\end{figure}

\section{Convergence Theory}
\label{sec:analysis}
In this section, we give a review of the literature on Federated Averaging and the associated convergence guarantees, presenting an analysis of how we expect these to be modified by the introduction of curriculum learning.

\paragraph{Curricula, Data Dissimilarity and Convergence}
The score of the data samples is based on the server's parameter vector $\theta_g$. Naturally, this approach creates a significant association between the degree of statistical dissimilarity of the data at each client with the training difficulty score used to rank data samples for curriculum learning. So we can safely purport that CL, for non-iid data, results in a level of dissimilarity that increases with the iteration $t$. 

To understand how this affects the convergence, we review a few standard works and study how increasing heterogeneity with the iteration number affects the convergence guarantees.

To begin with, the state of the art in convergence theory of Federated Averaging (or Local SGD) for convex objectives is given, to the best of our knowledge, in~\cite{khaled2020tighter}.

Here the main result of interest is~\cite[Theorem 5]{khaled2020tighter}, for which the objective optimality gap is bounded by,
\[
\mathbb{E}[f(\theta_T)-f(\theta^*)]\le \frac{C_1}{\gamma T}+C_2\gamma \sigma^2+C_3 \gamma^2 \sigma^2
\]
where the variance $\sigma$ is proportional to the heterogeneity. It can be seen from the convergence theory that the bound changes to, with $\sigma_t$ iteration dependent,
\[
\mathbb{E}[f(\theta_T)-f(x^*)]\le \frac{C_1}{\gamma T}+C_2\gamma \sum\limits_{t=0}^T\sigma_t^2+C_3 \gamma^2 \sum\limits_{t=0}^T\sigma_t^2
\]
suggesting an overall better convergence quality for any given iteration, since we expect $\sigma_t<\sigma$ up until $t=T$, i.e., early iterations generate better accuracy than otherwise.

In regards to nonconvex objectives, which are of course more faithful to the practice of training neural networks, to the best of our knowledge the state of the art in theoretical convergence guarantees for local SGD is given in~\cite{haddadpour2019convergence}. There, a notion of gradient similarity is presented,
\[
\Lambda(\theta,q) = \frac{\sum\limits_{m=1}^Mq_m \|\nabla f_m(\theta)\|^2}{\left\|\sum\limits_{m=1}^Mq_m \nabla f_m(\theta)\right\|^2}
\]
and assuming a bound $\lambda$ on this term, $\lambda$ does not appear directly in the convergence bounds in~\cite[Theorem 4.2 and Theorem 4.4]{haddadpour2019convergence} (respectively for the objective satisfying the PL condition and the general case). However, the number of local steps, which they denote as $E$, (i.e. the number of SGD steps in \texttt{ClientUpdate} in Algorithm~\ref{alg:fedpeft}) depends on $E\propto 1/\lambda$, meaning the greater the dissimilarity and the fewer local iterations are permitted to ensure convergence, a net increase in the total number of communications necessary.

The use of the FedProx objective can also be analyzed through the lens of iterate-varying dissimilarity. Considering~\cite[Theorem 4]{li2020federated} we have that with,
\[
\rho_t = \frac{1}{\mu}-\bar{\rho}(B_t,\gamma,\mu),\,\bar{\rho}(B_t,\gamma,\mu)=O(B_t)O(\gamma)O(1/\mu)
\]
and, with $S_t$ devices chosen at iteration $t$
\[
\mathbb{E}[f(x_{t+1}|S_t]-f(x_t)\le -\rho_t \|\nabla f(x_t)\|^2
\]
and thus with Curriculum training, we see increasing $B_t$ and thus decreasing $\rho_t$, and thus, again, we shall expect to see initial faster and then gradually slower convergence.

\paragraph{Local SGD Model and Convergence Analysis}
Consider the Local SGD framework as presented
in Algorithm~\ref{alg:localsgd}.

\begin{algorithm}
\caption{Local SGD Model of Algorithm~\ref{alg:fedpeft} }\label{alg:localsgd}

\small
\begin{algorithmic}
\item \hspace{-6mm}
\noindent \colorbox[rgb]{1, 0.95, 1}{
\begin{minipage}{0.9\columnwidth}

\textbf{Input:} $M$ clients indexed by $m$, participating-client number $Q$, communication rounds $T$, local optimization steps $J$, server model $f$ with parameters $\theta_g$

\end{minipage}
}
\item \hspace{-6mm}
\colorbox[gray]{0.95}{
\begin{minipage}{0.9\columnwidth}
\item  \textbf{Server executes:}

\item     \hspace*{\algorithmicindent} initialize $f$ with $\theta_g$

\item     \hspace*{\algorithmicindent} \textbf{for } each round $t=0,1,2,...,T$ \textbf{do}

\item     \hspace*{\algorithmicindent} \quad $\mathbb{S}_t \leftarrow$ (random set of $Q$ clients)

\item     \hspace*{\algorithmicindent} \quad \textbf{for} each client $q\in \mathbb{S}_t$ \textbf{in parallel do}

\item     \hspace*{\algorithmicindent} \quad \quad  broadcast $\theta^{t}_g $ to clients as $\theta^{(t,0)}_k$

\item     \hspace*{\algorithmicindent} \quad \quad \textbf{for} $j=0,1,2,...,J$

\item     \hspace*{\algorithmicindent} \quad \quad \quad   Sample $g_k^{(t,j)}\sim \nabla f(\theta_k^{(t,j)},\mathcal{D}_k)$

\item     \hspace*{\algorithmicindent} \quad \quad \quad   $\theta_{k}^{(t,j+1)} \leftarrow \theta_k^{(t,j)}-\alpha^{(t,j)} g_k^{(t,j)}$

\item     \hspace*{\algorithmicindent} \quad \quad $ \theta^{(t+1)}_g= \sum_{k=1}^K \frac{|\mathcal{D}_k|}{\sum_{i=1}^K|\mathcal{D}_i|}\theta_k^{(t,J)}$ 

\item     \hspace*{\algorithmicindent} \textbf{return} $\theta_g^{t+1}$
\end{minipage}
}
\end{algorithmic}
\end{algorithm}

We formalize the notion of distributional skew by making the following assumption on the bias structure associated with each stochastic gradient computation:
\begin{assumption}\label{as:bias}
It holds that $g_k^{(t,j)}$ satisfies,
\begin{equation}\label{eq:biasedsgd}
    g_k^{(t,j)} = \nabla f(\theta^{(t,j)}_k)+b_k^{(t,j)}(\theta^{(t,j)}_k)+n_k^{(t,j)}(\theta^{(t,j)}_k,\xi^{(t,j)}_k)
\end{equation}
where $\|b_k^{(t,j)}(\theta^{(t,j)}_k)\|^2\le B^{(t,j)}$ for all $k$, and, for all $\theta$, 
\begin{equation}\label{eq:sumbiaszero}
    \sum\limits_{k\in S_t} b_k^{(t,j)}(\theta) = 0
\end{equation}
and $\xi^{(t,j)}_k$ is a random variable satisfying,
\begin{equation}\label{eq:addunbiasnoise}
    \mathbb{E}_{\xi}[n_k^{(t,j)}(\theta^{(t,j)}_k,\xi^{(t,j)}_k)] = 0
\end{equation}
\end{assumption}
We note that,
\[
\begin{array}{l}
B^{(0,0)}=B^{(0,1)}=...=B^{(0,J)}< 
B^{(1,0)}=B^{(1,1)}=...\\ =B^{(1,J)} < ...
< B^{(t,0)}=B^{(t,1)}=...=B^{(t,J)}
< B^{(t+1,0)} \\ =B^{(t+1,1)}=...=B^{(t+1,J)}
< ... < B^{(T,0)}=B^{(T,1)} \\ =...=B^{(T,J)}
\end{array}
\]
for \textbf{client based} curriculum training, and
\[
\begin{array}{l}
B^{(0,0)}<B^{(0,1)}<...=B^{(0,J)}= 
B^{(1,0)}<B^{(1,1)}<...\\ <B^{(1,J)} = B^{(2,0)} <  ...
 B^{(t,0)}<B^{(t,1)}<...\\ <B^{(t,J)}
= B^{(t+1,0)} <B^{(t+1,1)}<...\\ <B^{(t+1,J)}
 ... B^{(T,K-1)}< B^{(T,J)}
\end{array}
\]
for \textbf{data based} curriculum training.

Now we present two results as depending on the conditions applying to the functions characterizing the optimization. In the first case, we shall consider strongly convex objectives, as characterizing least squares empirical risk minimization of, e.g., linear models. In this scenario, we permit the variance to grow with the parameter size, i.e., we do not assume bounded gradients.
\begin{theorem}\label{th:convconvex}
Assume that
\begin{itemize}
    \item $f$ is strongly convex with convexity parameter $\mu > 0$ 
    \item $\nabla f$ is Lipschitz continuous with Lipschitz constant $L$
    \item the noise variance satisfies,
    \[
    \begin{array}{l}
    \mathbb{E}_{\xi}\left[\left\|n_k^{(t,j)}(\theta^{(t,j)}_k,\xi^{(t,j)}_k)\right\|^2\right] \\ \qquad \le M\left\|\nabla f(\theta^{(t,j)}_k)+b_k^{(t,j)} (\theta^{(t,j)}_k)\right\|^2+\sigma^2
    \end{array}
    \]
    \item For all $t,j$ we have ,
    \[
    \alpha^{(t,j)}\le \frac{1}{4(3+2M)L}
    \]
\end{itemize}
Then it holds that the distance to the solution satisfies, after each averaging step,
\[
\begin{array}{l}
\mathbb{E}\left\|\hat{\theta}^{(T,0)}-\theta^*\right\|^2 \le 
\prod\limits_{t=1}^T 
\prod\limits_{j=0}^J (1-\alpha^{(t,j)}\mu/2)\|\hat{\theta}^{(0,0)}-\theta^*\|^2\\
\qquad +\sum\limits_{t=1}^T 
\sum\limits_{j=0}^J \frac{2(\alpha^{(t,j)})^2[ L ((3+2M)B^{(t,j)}+3\sigma^3]}{Q} \\
\qquad +\sum\limits_{t=1}^T 
\sum\limits_{j=0}^J\frac{2\alpha^{(t,j)} L (B^{(t,j)})^2/\mu)}{Q}
\end{array}
\]
\end{theorem}
In studying the form of this result, we note that the overall convergence
rate and error resembles the original with an important caveat in regards to the 
error on account of the bias term. First, the bias term adds an error proportional to
the stepsize, thus yielding an asymptotic error bounded from below with the bias.
Second, the stepsize can be used to mitigate the error from the bias terms. Indeed, with,
e.g., data-based curriculum, if $B^{(t,j)}=O(j^{1/4})$ then $\alpha^{(t,j)}=O(t^{-1}j^{-1/4})$ would mitigate the growing error. It is clear that the standard practice of diminishing stepsizes will result in a lower total error at each iteration for curriculum compared to anti-curriculum. Standard Local SGD guarantees are not preserved regardless, however, with the asymptotic bias depending on the total degree of data heterogeneity, summed in this weighted manner throughout the optimization procedure.

\textbf{Nonconvex Objectives}
Now we consider the general case of nonconvex objectives without any additional conditions regarding the growth properties of the objective function to permit generality encompassing the functional properties of neural networks. Using the biased SGD framework and inspired by the structure of the convergence theory of~\cite{zhou2018convergence}, we study the effect of the associated gradient estimate errors.
\begin{theorem}\label{th:convnonconvex}
Assume that
\begin{itemize}
    \item $\|\nabla f\|$ is uniformly bounded by $G$
    \item $\nabla f$ is Lipschitz continuous with Lipschitz constant $L$
    \item the noise variance satisfies,
    \[
    \mathbb{E}_{\xi}\left[\left\|n_k^{(t,j)}(\theta^{(t,j)}_k,\xi^{(t,j)}_k)\right\|^2\right] \le \sigma^2
    \]
    \item $f$ is lower bounded by $f_*$
\end{itemize}
Then we obtain the ergodic rate,
\[
\begin{array}{l}
\sum\limits_{t=0}^T\sum\limits_{j=0}^J \|\nabla f(\theta^{(t,0)}\|^2 \le Q(f(\theta^0)-f^*)\\ \qquad +2\sum\limits_{t=0}^T\sum\limits_{j=0}^J \alpha^{(t,j)}\left(\alpha^{(t,j)}+\sum\limits_{l=j}^J \alpha^{(t,l)}\right) LG^2
\end{array}
\]
\end{theorem}
Compared to standard results, we can see that curricula  contributes an error that corresponds to the cross terms of the stepsizes, indicating a benefit to annealing the stepsize along local iterations as well as along averaging steps. 

Now we present the proofs that build up the argument for the main new convergence results we present in Section~\ref{sec:analysis}, specifically Theorem~\ref{th:convconvex} and~\ref{th:convnonconvex}.
\subsection{Strongly Convex Problems}
\begin{lemma}\label{lem:convex1}
The stochastic gradient second moment satisfies:
\[
\begin{array}{l}
\mathbb{E}[\|g^{(t,j)}_k\|^2] \le  2(3+2M)L(f(\theta^{(t,j)}_k)-f^*)\\
\qquad \qquad \qquad +(3+2M)B^{(t,j)}+3\sigma^2
\end{array}
\]
\end{lemma}
\begin{proof}
Follows from,
\[
\begin{array}{l}
\mathbb{E}[\|g^{(t,j)}_k\|^2]\le 3\|\nabla f(\theta^{(t,j)}_k)\|^2+3\|b^{(t,j)}_k(\theta^{(t,j)}_k)\|^2\\ \qquad\qquad +3\mathbb{E}[\|n^{(t,j)}_k(\theta^{(t,j)}_k,\xi^{(t,j)}_k)\|^2]  \\
\le (3+2M)\|\nabla f(\theta^{(t,j)}_k)\|^2+(3+2M)B^{(t,j)}+3\sigma^2 \\
\le 2(3+2M)L(f(\theta^{(t,j)}_k)-f^*)+(3+2M)B^{(t,j)}+3\sigma^2
\end{array}
\]
where in the last line we used $\nabla f(\theta^*)=0$ and $L$-smoothness.
\end{proof}
Define,
\[
g^{(t,j)} = \frac{1}{|S_t|}\sum\limits_{k\in S_t} g^{(t,j)}_k,\,\,
\bar{g}^{(t,j)} = \frac{1}{|S_t|}\sum\limits_{k\in S_t} \nabla f(x^{(t,j)}_k)
\]
Note that Assumption~\ref{as:bias}, in particular~\eqref{eq:addunbiasnoise} and~\eqref{eq:addunbiasnoise} imply that $\mathbb{E}[g^{(t,j)}_k]=\bar{g}^{(t,j)}$.

The next Lemma is similar to~\cite[Lemma 5]{khaled2020tighter} with a simpler proof of simple adding the terms across agents up using the previous result.
\vspace{-5mm}
\begin{lemma}\label{lem:convex2}
\[
\begin{array}{l}
\mathbb{E}[\|g^{(t,k)}-\bar{g}^{(t,j)}\|^2]\le \\
\frac{(3+2M)}{Q^2}\sum\limits_{k\in S^{(t)}}\left[2L(f(\theta^{(t,j)}_k)-f^*)+B^{(t,j)}\right]+\frac{3\sigma^2}{Q}
\end{array}
\]
\end{lemma}
The next Lemma is similar to~\cite[Lemma 6]{khaled2020tighter} which in turn follows~\cite[Lemma 2.1]{stich2019local}. Consider the sequence,
\[
\hat{\theta}^{(t,j+1)}=\hat{\theta}^{(t,j)}-\alpha^{(t,j)}g^{(t,j)}
\]
and note that by this construction $\hat{x}^{(t,J)}=\hat{x}^{(t+1,0)}$. The proof is a straightforward application of strong convexity. It holds that,
\begin{lemma}\label{lem:convex3}
\[
\begin{array}{l}
\|\hat{\theta}^{(t,j)}-\alpha^{(t,j)}\bar{g}^{(t,j)}-\theta^*\|^2\le \|\hat{\theta}^{(t,j)}-\theta^*\|^2\\+\frac{2\alpha^{(t,j)}}{Q}\sum\limits_{k\in S^{(t)}}
\left[(\alpha^{(t,j)} L-1/2)(f(\theta^{(t,j)}_k)-f(\theta^*))\right.\\\qquad\qquad\qquad\qquad \left.-\frac{\mu}{2}\|\theta^{(t,j)}_k-\theta^*\|^2\right] \\ \quad +\frac{2\alpha^{(t,j)} L}{Q}\sum\limits_{k\in S^{(t)}}\|\hat{\theta}^{(t,j)}-\theta^{(t,j)}_k\|^2
\end{array}
\]
\end{lemma}

Finally we obtain our first derivation of expected convergence below. 
\begin{lemma}\label{lem:convex4}
Let $\bar{L}:=( L+(3+2M)2L/Q)$ and assume that $\alpha^{(t,j)}\le \frac{1}{4\bar{L}}$. It holds that the expected distance of the average parameter to the solution satisfies the recursion,
\[
\begin{array}{l}
\mathbb{E}[\|\hat{\theta}^{(t,j+1)}_k-\theta^*\|^2] \le \left(1-\alpha^{(t,j)}\mu\right) \|\hat{\theta}^{(t,j)}-\theta^*\|^2  \\
\qquad-\frac{\alpha^{(t,j)}}{2}\left(f(\hat{\theta}^{(t,j)}-f(\theta^*)\right)\\ \qquad
+\frac{2\alpha^{(t,j)} L}{Q}\sum\limits_{k\in S^{(t)}}\|\hat{\theta}^{(t,j)}-\theta^{(t,j)}_k\|^2 \\ \qquad 
+\frac{(3+2M)(\alpha^{(t,j)})^2B^{(t,j)}}{Q}+\frac{3\sigma^2(\alpha^{(t,j)})^2}{Q} 
\end{array}
\]
\end{lemma}
\begin{proof}
Using the previous set of results,
\[
\begin{array}{l}
\mathbb{E}[\|\hat{\theta}^{(t,j+1)}_k-\theta^*\|^2] \le \|\hat{\theta}^{(t,j)}-\alpha^{(t,j)}\bar{g}^{(t,j)}-\theta^*\|^2 \\
\qquad\qquad+(\alpha^{(t,j)})^2\mathbb{E}[\|g^{(t,j)}-\bar{g}^{(t,j)}\|^2] \\
\le \|\hat{\theta}^{(t,j)}-\theta^*\|^2\\\quad +\frac{2\alpha^{(t,j)}}{Q}\sum\limits_{k\in S^{(t)}}
\left[(\alpha^{(t,j)} L-1/2)(f(\theta^{(t,j)}_k)-f(\theta^*))\right.\\\qquad\qquad\qquad\qquad \left.-\frac{\mu}{2}\|\theta^{(t,j)}_k-\theta^*\|^2\right] \\ \quad +\frac{2\alpha^{(t,j)} L}{Q}\sum\limits_{k\in S^{(t)}}\|\hat{\theta}^{(t,j)}-\theta^{(t,j)}_k\|^2 \\ \quad 
+\frac{(3+2M)(\alpha^{(t,j)})^2}{Q^2}\sum\limits_{k\in S^{(t)}}\left[2L(f(\theta_k^{(t,j)}-f(\theta^*))+B^{(t,j)}\right]\\ \qquad +\frac{3\sigma^2(\alpha^{(t,j)})^2}{Q} \\
\le \|\hat{\theta}^{(t,j)}-\theta^*\|^2\\\quad +\frac{2\alpha^{(t,j)}}{Q}\sum\limits_{k\in S^{(t)}}
\left[(\alpha^{(t,j)}\bar{L}-1/2)(f(\theta^{(t,j)}_k)-f(\theta^*))\right.\\\qquad\qquad\qquad\qquad \left.-\frac{\mu}{2}\|\theta^{(t,j)}_k-\theta^*\|^2\right] \\ \quad +\frac{2\alpha^{(t,j)} L}{Q}\sum\limits_{k\in S^{(t)}}\|\hat{\theta}^{(t,j)}-\theta^{(t,j)}_k\|^2 \\ \quad 
+\frac{(3+2M)(\alpha^{(t,j)})^2B^{(t,j)}}{Q}+\frac{3\sigma^2(\alpha^{(t,j)})^2}{Q} 
\end{array}
\]
By assumption $\alpha^{(t,j)}\bar{L}-1/2\le -\frac{1}{4}$ and then applying Jensen's inequality we have $\frac{1}{Q}\sum\limits_{k\in S^{(t)}}\left[-\frac{1}{4}(f(\theta^{(t,j)}_k)-f(\theta^*))-\frac{\mu}{2}\|\theta^{(t,j)}_k-\theta^*\|^2\right]\le -\left(\frac{1}{4}(f(\hat{\theta}^{(t,j)})-f(\theta^*))+\frac{\mu}{2}\|\hat{\theta}^{(t,j)}-\theta^*\|^2\right)$. Plugging this expression into the last displayed equation, the conclusion follows.
\end{proof}

Next, from Lemma~\ref{lem:convex1} we can conclude that
\begin{equation}\label{eq:varggk}
\begin{array}{l}
\frac{1}{Q}\sum\limits_{k\in S^{(t)}} \mathbb{E}\left[   \left\|g^{(t,j)}_k-g^{(t,j)}\right\|^2\right]\\  \le \frac{2(3+2M)L}{Q}\sum\limits_{k\in S^{(t)}} \left(f(\theta_k^{(t,j)})-f(\theta^*)\right) \\
\qquad +(3+2M)B^{(t,j)}+ 3\sigma^2
\end{array}
\end{equation}
Next we derive a recursion on the average parameter deviation.
\begin{lemma}
Let $\mu>0$. The average iterate deviation satisfies the bound,
\[
\begin{array}{l}
\mathbb{E}\left[\frac{1}{Q}\sum\limits_{k\in S^{(t)}}\|\hat{\theta}^{(t,j+1)}-\theta^{(t,j+1)}_k\|^2\right] \\
\le (1-\alpha^{(t,j)}\mu/2)\mathbb{E}\left[\frac{1}{Q}\sum\limits_{k\in S^{(t)}}\|\hat{\theta}^{(t,j)}-\theta^{(t,j)}_k\|^2\right] \\
+ \frac{2(3+2M)L(\alpha^{(t,j)})^2}{Q}\sum\limits_{k\in S^{(t)}} \left(f(\theta_k^{(t,j)})-f(\theta^*)\right) \\
+\alpha^{(t,j)}\left(\alpha^{(t,j)}(3\hspace{-1mm}+\hspace{-1mm}2M)\hspace{-1mm}+\hspace{-1mm}B^{(t,j)}/\mu\right)B^{(t,j)}+\hspace{-1mm} 3(\alpha^{(t,j)})^2\sigma^2
\end{array}
\]
\end{lemma}
\begin{proof}
Indeed, compute directly,
\[
\begin{array}{l}
\mathbb{E}\left[\frac{1}{Q}\sum\limits_{k\in S^{(t)}}\|\hat{\theta}^{(t,j+1)}-\theta^{(t,j+1)}_k\|^2\right] \\
\le \mathbb{E}\left[\frac{1}{Q}\sum\limits_{k\in S^{(t)}}\|\hat{\theta}^{(t,j)}-\theta^{(t,j)}_k\|^2\right] \\ 
\qquad +\frac{(\alpha^{(t,j)})^2}{Q}\sum\limits_{k\in S^{(t)}}\mathbb{E}\left[\|g^{(t,j)}-g^{(t,j)}_k\|^2\right] \\ 
\qquad -\frac{2\alpha^{(t,j)} }{Q}\sum\limits_{k\in S^{(t)}}\mathbb{E}\left[\left\langle \theta^{(t,j)}_k-\hat{\theta}^{(t,j)},g^{(t,j)}_k-g^{(t,j)}\right\rangle\right]
\end{array}
\]
For the second term in the above expression we can apply~\eqref{eq:varggk}. For the third, we note that,
\[
\begin{array}{l}
-\sum\limits_{k\in S^{(t)}}\mathbb{E}\left[\left\langle \theta^{(t,j)}_k-\hat{\theta}^{(t,j)},g^{(t,j)}_k-g^{(t,j)}\right\rangle\right] \\
= -\sum\limits_{k\in S^{(t)}}\left\langle \theta^{(t,j)}_k-\hat{\theta}^{(t,j)},\nabla f(\theta^{(t,j)}_k)+b_k^{(t,j)}(\theta^{(t,j)}_k)\right\rangle\\\hspace{-1mm} +\hspace{-1mm}\sum \limits_{k\in S^{(t)}}\left\langle\theta^{(t,j)}_k \hspace{-1mm}-\hspace{-.7mm}\hat{\theta}^{(t,j)},\frac{1}{Q}\sum\limits_{k\in S^{(t)}}\left[\nabla f(\theta^{(t,j)}_k)\hspace{-.7mm}+\hspace{-.7mm}b_k^{(t,j)}(\theta^{(t,j)}_k)\right]\right\rangle \\ 
= \sum\limits_{k\in S^{(t)}}\left\langle \hat{\theta}^{(t,j)}-\theta^{(t,j)}_k,\nabla f(\theta^{(t,j)}_k)+b_k^{(t,j)}(\theta^{(t,j)}_k)\right\rangle \\
\le \sum\limits_{k\in S^{(t)}} \left[f(\hat{\theta}^{(t,j)})-f(\theta^{(t,j)}_k)-\frac{\mu}{2}\|\theta^{(t,j)}_k-\theta^{(t,j)}\|^2\right]\\ \qquad +\sum\limits_{k\in S^{(t)}}
B^{(t,j)} \|\theta^{(t,j)}_k-\theta^{(t,j)}\|
\end{array}
\]
where we used strong convexity in the inequality. Applying Young's inequality to obtain $B^{(t,j)} \|\theta^{(t,j)}_k-\theta^{(t,j)}\|\le \frac{\mu}{4}\|\theta^{(t,j)}_k-\theta^{(t,j)}\|^2+\frac{1}{\mu}(B^{(t,j)})^2$ yields the final result.
\end{proof}
Now we want to use the previous Lemma in order to bound the contribution of the average iterate discrepancy to the overall descent appearing in Lemma~\ref{lem:convex4}. Taking a sum for a given $t$, for $j=1,...,J$, we can see that 
\begin{equation}\label{eq:sumavgdiffj}
    \begin{array}{l}
    \sum\limits_{j=0}^J\mathbb{E}\left[\frac{1}{Q}\sum\limits_{k\in S^{(t)}}\|\hat{\theta}^{(t,j)}-\theta^{(t,j)}_k\|^2\right] \\
    \le \sum\limits_{j=0}^J\frac{2\alpha^{(t,j)} L}{Q}\prod\limits_{l=j}^J (1-\alpha^{(t,l)}\mu/2) \\ 
    \qquad \left[2\alpha^{(t,j)}(3+2M)L\sum\limits_{k\in S^{(t)}} \left(f(\theta_k^{(t,j)})-f(\theta^*)\right) \right.\\
\qquad\quad \left.+\left(\alpha^{(t,j)}(3+2M)+B^{(t,j)}/\mu\right)B^{(t,j)}\right.\\ \qquad\qquad\qquad+\left. 3\alpha^{(t,j)}\sigma^2\right]
    \end{array}
\end{equation}
With that, we proceed with the main result:

\begin{proof}\textbf{ of Theorem}~\ref{th:convconvex}
From Lemma~\ref{lem:convex4} and~\eqref{eq:sumavgdiffj}
\[
\begin{array}{l}
\mathbb{E}[\|\hat{\theta}^{(t+1,0)}-\theta^*\|^2] \le \prod\limits_{j=0}^J (1-\alpha^{(t,j)}\mu/2) \|\hat{\theta}^{(t,0)}-\theta^*\|^2  \\
\qquad +\sum\limits_{j=0}^J\frac{2\alpha^{(t,j)} L}{Q}\prod\limits_{l=j}^J (1-\alpha^{(t,l)}\mu/2) \\ 
    \qquad \left[\left(2\alpha^{(t,j)}(3+2M)L-\frac{1}{2}\right)\sum\limits_{k\in S^{(t)}} \left(f(\theta_k^{(t,j)})-f(\theta^*)\right) \right.\\
\qquad\quad \left.+\left(2\alpha^{(t,j)}(3+2M)+B^{(t,j)}/\mu\right)B^{(t,j)}\right.\\ \qquad\qquad\qquad+\left. 6\alpha^{(t,j)}\sigma^2\right]
\end{array}
\]
Noting that the assumption on the Theorem implies that the term involving the objective value difference is negative, we obtain the statement of the main result.
\end{proof}

\subsection{Nonconvex Objectives}
\begin{proof}of Theorem~\ref{th:convnonconvex}
As standard, we begin by applying the Descent Lemma across subsequent
averaging steps. 
\[
\begin{array}{l}
f(\theta^{(t+1,0)})-f(\theta^{(t,0)})\le \left\langle \nabla f(\theta^{(t,0)}),\theta^{(t+1,0)}-\theta^{(t,0)}\right\rangle \\\qquad\qquad +\frac{L}{2}\|\theta^{(t+1,0)}-\theta^{(t,0)}\|^2 \\
\quad \le -\left\langle \nabla f(\theta^{(t,0)}),\frac{1}{Q}\sum\limits_{k\in S^{(t)}}\sum\limits_{j=0}^J \alpha^{(t,j)} g^{(t,j)}_k\right\rangle \\
\qquad +\frac{L}{2}\left\|\frac{1}{Q}\sum\limits_{k\in S^{(t)}}\sum\limits_{j=0}^J \alpha^{(t,j)}  g^{(t,j)}_k\right\|^2
\end{array}
\]
Now, we consider the discrepancy of $g^{(t,j)}_k$ to $\nabla f(\theta^{(t,0)})$
to obtain a perturbation from the decrease we expect to get, that we wish to eventually bound relative to said decrease. Specifically, taking total expectations (and implicitly using the tower property):
\vspace{-5mm}
\[
\begin{array}{l}
\mathbb{E}\left[\frac{1}{Q}\sum\limits_{k\in S^{(t)}}\sum\limits_{j=0}^J g^{(t,j)}_k\right] \\ =  \frac{1}{Q}\sum\limits_{k\in S^{(t)}}\sum\limits_{j=0}^J\alpha^{(t,j)} \mathbb{E}\left[\nabla f(\theta^{(t,0)})+b_k^{(t,0)}(\theta^{(t,0)})\right.\\ \qquad 
-f(\theta^{(t,0)})-b_k^{(t,0)}(\theta^{(t,0)})+\nabla f(\theta^{(t,0)},\xi^{(t,j)}_k) \\
\qquad - \left.\nabla f(\theta^{(t,0)},\xi^{(t,j)}_k) +\nabla f(\theta^{(t,j)},\xi^{(t,j)}_k) \right] \\
= \frac{1}{Q}\sum\limits_{k\in S^{(t)}}\sum\limits_{j=0}^J\alpha^{(t,j)} \left[\nabla f(\theta^{(t,0)})
\right. \\
\qquad \left. -\mathbb{E}\left[\nabla f(\theta^{(t,0)},\xi^{(t,j)}_k) +\nabla f(\theta^{(t,j)},\xi^{(t,j)}_k)\right]\right]
\end{array}
\]
and so, combining the previous two sets of equations,\vspace{-5mm}
\[
\begin{array}{l}
f(\theta^{(t+1,0)})-f(\theta^{(t,0)})\le -\frac{\sum\limits_{j=0}^J \alpha^{(t,j)}}{Q} \|\nabla f(\theta^{(t,0)})\|^2 \\
\qquad + \frac{\sum\limits_{j=0}^J\left(\alpha^{(t,j)}\sum\limits_{l=j}^J \alpha^{(t,l)}\right)}{Q}L G^2 \\ 
\qquad + \frac{\sum\limits_{j=0}^J\left(\alpha^{(t,j)}\right)^2L G^2}{2Q}
\end{array}
\]
from which we obtain the final result.
\end{proof}

\section{Effect of amount of data on clients end} \label{Effect-of-amount-of-data}
In this section, we are interested in understanding whether the previous conclusions we made for CIFAR10 generalize to both high and low data regimes on the client's end. In particular, we divide the larger dataset into multiples of the number of clients and randomly assign $M$ of those data partitions to the $M$ clients. The larger the number of partitions, the smaller the amount of data on each of the clients. As can be seen from Fig.~\ref{fig:cifar10-partition} the amount of data that each client owns has no relationship with the benefit it gains from curriculum learning. In fact, CL ameliorates the classification accuracy performance equally under both lower and higher data regimes on the clients' end.


\begin{figure}[H]
\centering
\begin{subfigure}{0.48\linewidth}
\includegraphics[width=1\linewidth]{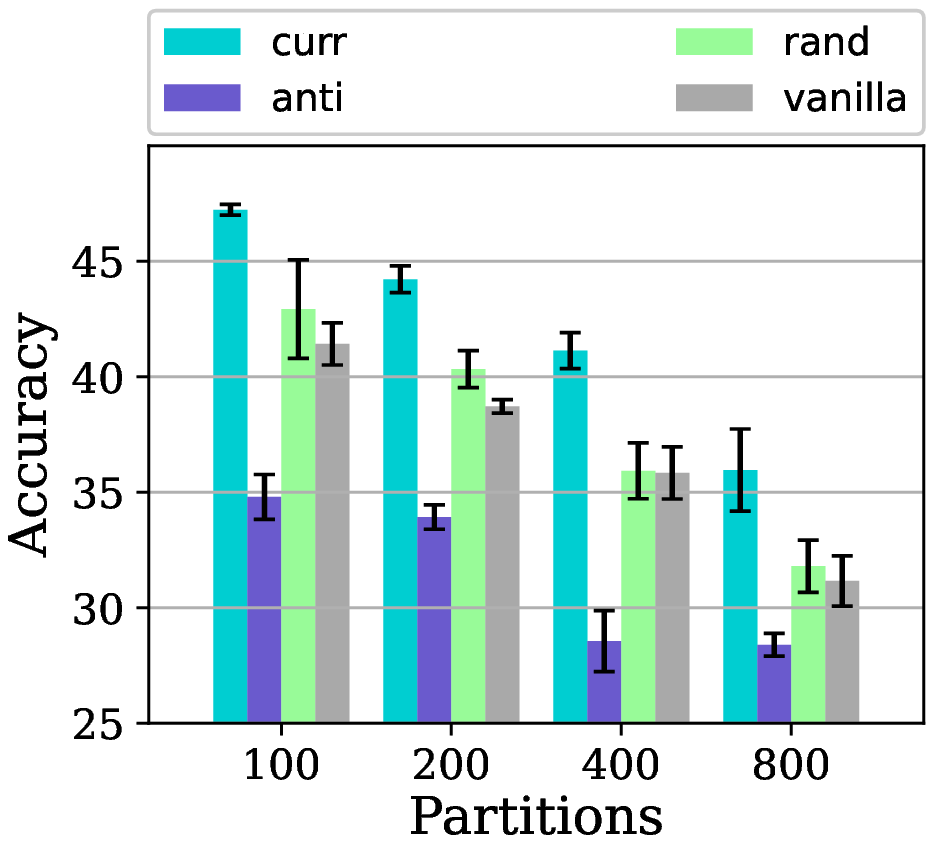}    
\end{subfigure}
\begin{subfigure}{0.48\linewidth}
\includegraphics[width=1\linewidth]{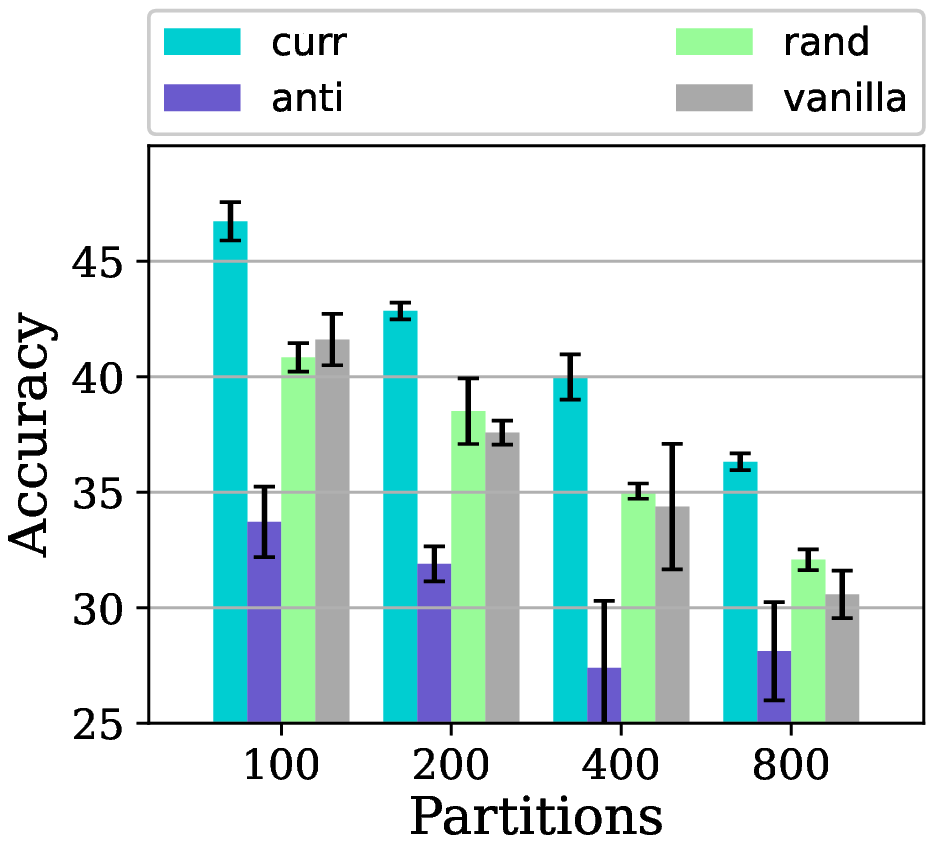}   
\end{subfigure}
\vspace{-4mm}
\caption{\footnotesize \textbf{There is no correlation between the amount of data on the client's end and the benefit they gain from ordered learning. The accuracy decreases when the amount of data each client owns is reduced, but it gains the same amount of benefit from curriculum learning with more data.}  Evaluating the impact of the amount of data each client owns on the accuracy when the clients employ curriculum, anti-curriculum or random ordering during their local training on CIFAR-10 with Non-IID (2) for FedAVg (left), and with Dir(0.05) for Fedprox (right).  All curricula use the linear pacing functions with a = 0.8 and b = 0.2. Each experiment is repeated three times for a total of 100 communication rounds with ten local epochs, and the mean and standard deviation for global test accuracy are reported.}
\label{fig:cifar10-partition}
\vspace{-4mm}
\end{figure}

\begin{figure*}[htbp]
\centering

\begin{subfigure}{0.23\linewidth}
    \includegraphics[width=1\linewidth]{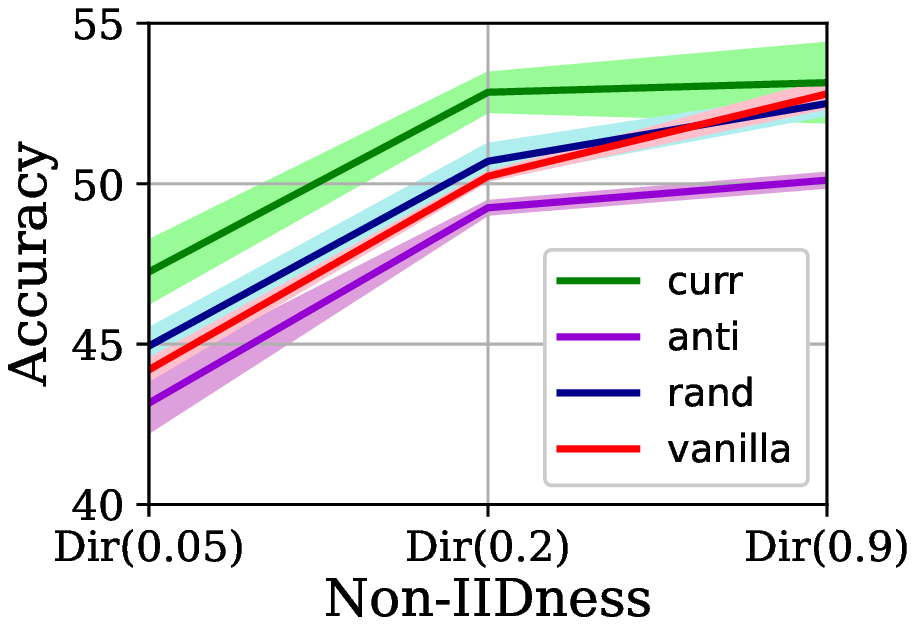}
\end{subfigure}
\begin{subfigure}{0.23\linewidth}
    \includegraphics[width=1\linewidth]{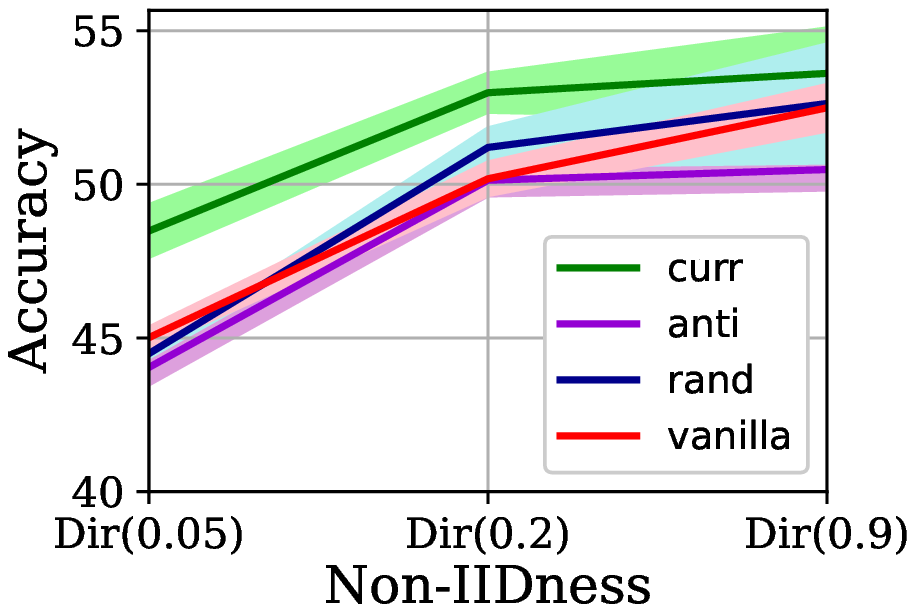}
\end{subfigure}
\begin{subfigure}{0.23\linewidth}
    \includegraphics[width=1\linewidth]{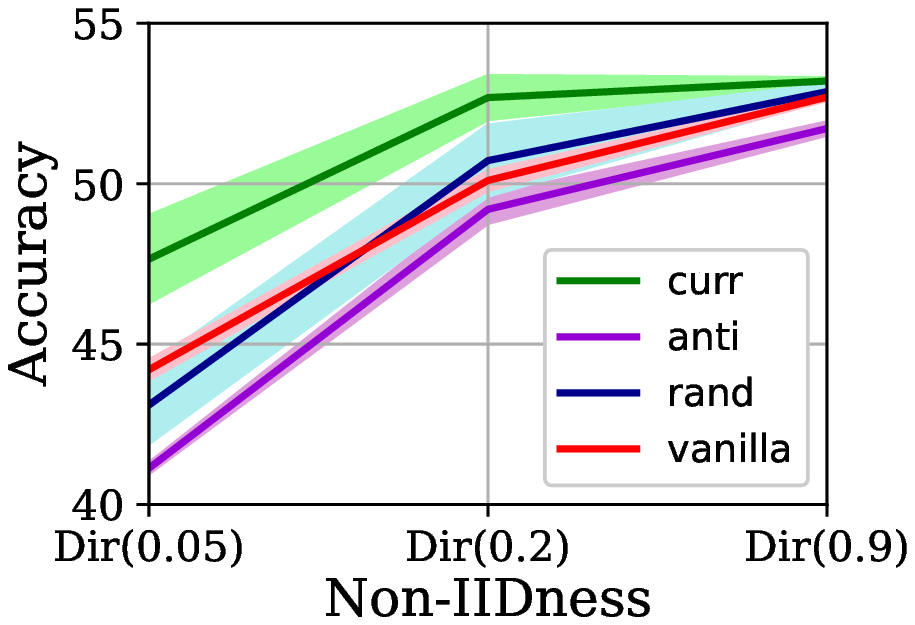}
\end{subfigure}
\begin{subfigure}{0.23\linewidth}
    \includegraphics[width=1\linewidth]{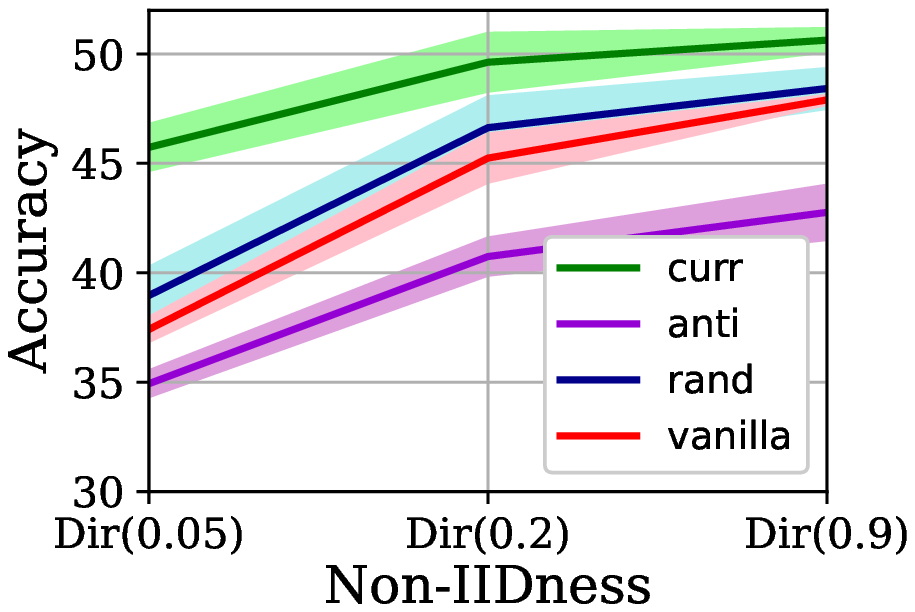}
\end{subfigure}
\vspace{-3mm}
\caption{\footnotesize \textbf{Curriculum-learning helps more when training with more severe data heterogeneity across clients on CIFAR-100.} Test accuracy of different baselines when sweeping from 
extremely Non-IID setting, Dir (0.05) to highly IID setting, Dir(0.9). For each baseline, the average of final global test
accuracy is reported. We run each baseline 3 times
for 100 communication rounds with 10 local epochs. The figures from left to right, are for FedAvg, Fedprox, Scaffold, and FedNova baselines.}
\label{fig:cifar100-niid}

\end{figure*}

\section{Effect of pacing function and its parameters in IID and Non-IID FL} \label{Effect-of-pacing-function-sup}
This subsection complements subsection~\ref{Effect-of-pacing-function} of the main paper, where we evaluated the effect of pacing function and its hyperparameter $a$ when clients train on CIFAR-10 with FedAvg under IID data. Here, we report the results for FedAvg under Non-IID Dir($0.05$).  The conclusion is similar--Fig.~\ref{pacing-a-cifar10-iid-ordering} shows that bigger values of $a$ provide better accuracy performance for most of the pacing function families on both extreme IID and Non-IID setting. It is noteworthy that, the observations generalize to other baselines, as discussed in different sections of the paper.


\begin{figure}[H]
    \includegraphics[width=1\linewidth]{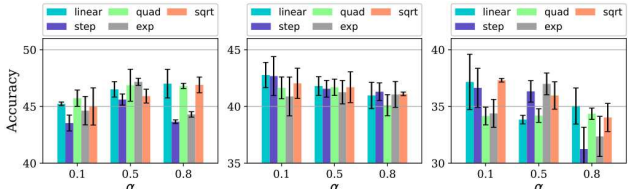}  
    \vspace{-8mm}
\caption{\footnotesize \textbf{ Bigger $a$ values provide better accuracy performance for most of pacing function families and on both IID and Non-IID setting for curriculum learning. But a notable contrast can be seen with random-/anti ordering.} The effect of using different pacing function families and their hyperparameter $a$ on the accuracy when the clients employ curriculum, anti-curriculum or random ordering during their local training on CIFAR-10 with Non-IID Dir(0.05) data. The figures from left to right are for curriculum, random, and anti ones.}
\label{pacing-a-cifar10-iid-ordering}
\end{figure}

\section{Effect of level of heterogeneity} \label{cifar-100-effect-of-level-hetero}
This subsection complements subsection~\ref{effect-of-hetero} of the main paper. In this section, we present further experimental results showing the relationship between ordering-based learning and the level of statistical data heterogeneity. Herein, we are interested in investigating whether the previous conclusions we made for CIFAR-10 generalize to other datasets such as CIFAR-100. Fig.~\ref{fig:cifar100-niid}  shows the same trend as in CIFAR-10, i.e., \emph{again, we see that as the data from the clients becomes more heterogeneous, the global model benefits more from curriculum learning, resulting in higher performance accuracy when compared to "vanilla" and "anti-/random" learning. } We provided rigorous analysis to explain this phenomenon.

\section{Related Work}
\label{sec:related-work}
Early CL formulated the easy-to-hard training paradigm in the context of deep learning \cite{bengio-curriculum}. CL determines
a sequence of training instances, which in essence corresponds to a list of samples ranked in ascending order of learning difficulty~\cite{SPL-curriculum-2015}. Samples are ranked according to  per-sample loss~\cite{curriculum-DIH-bilmes-2020}. In the early steps of training, samples with smaller loss (higher score) are selected, and gradually the subset size over time is increased to cover all the training data.  \cite{curriculum-humna-score-2015} proposed to manually sort the samples using human annotators. Self-paced learning (SPL)~\cite{SPL-curriculum-2015} chooses the curriculum based on hardness (e.g., per-sample loss) during training. \cite{curriculum-cscore-2020} proposes using a consistency score (c-score) calculated based on the consistency of a model in correctly predicting a particular example’s label trained on i.i.d. draws of the training set. \cite{curriculum-scoring-2018} determines the difficulty of learning an example by the metric of the earliest training iteration, after which the model predicts the ground truth class for that example in all subsequent iterations.

 \section{Implementation Details}
\label{impelement-detail}


We begin by splitting the dataset into $K$ partitions, and these partitions are distributed among the $N$ clients in the federation. For most experiments $M=100$ and the partitions are constructed with an input Non-IID Dirichlet distribution with parameter $\beta$ and using Algorithm~\ref{alg:partiiton_diff_dist} with $f_{ord} = 0$, unless otherwise specified. The merits of the Algorithm~\ref{alg:partiiton_diff_dist} are detailed in Section~\ref{sec:difficulty_based_partitioning}.

At the client, we use an SGD optimizer for training with an exponentially decaying learning $\eta =  \eta_{0} (1+ \alpha*i)^{-b}$, with parameters $\eta_{0}=0.001$, $\alpha=0.001$, $b=0.75$ and $i$ is the step index, and a momentum $\rho = 0.9$ and weight decay of $\omega = 5*10^{-4}$. The step count $i$ is a parameter local to the clients and is reset at the beginning of each federation round thereby resetting the learning rate back to $\eta_{0}$ for each round of federation. For the ResNet models however, we do not use the exponential decay learning rate and set $b=0$ with $\eta_{0}=0.01$, and weight decay $\omega = 0$, due to our observation that these values empirically work well.

A small batch size of $bs_{data} = 10$ is used on the server. At each client, we use the local epochs $n_{epoch} = 10$, which,
together with the client data partition size, determines the number of local steps at the clients between two global model averaging steps of the federation algorithm. The number of communication rounds of federation is $R = 100$ and the client participation rate is $f = 0.1$, unless otherwise specified. Similarly, when performing client curriculum, we use a client batch size of $bs_{client} = 10$.

Certain federated learning algorithms require additional algorithm specific parameters; these are chosen to match the best values reported by the authors in their respective papers. For reproducibility of the experiments, we seed our random number generator with a seed of $202207$ at the beginning of each experiment. Each experiment consists of $3$ trials, and we report the mean and variance of the results.

{\small
\bibliographystyle{unsrt}
\bibliography{Saeed}
}

\end{document}